%% file: camera_ready.tex
\documentclass[10pt,twocolumn,letterpaper]{article}

\usepackage{wacv}
\usepackage{times}
\usepackage{epsfig}
\usepackage{graphicx}
\usepackage{amsmath}
\usepackage{amssymb}
\usepackage{indentfirst}

\usepackage{bbold}
\usepackage{dsfont}

\usepackage{enumerate}
\usepackage[toc,page]{appendix}
\usepackage[font=small,labelfont=bf]{caption}

\usepackage{xcolor}
\usepackage{amsfonts}

\usepackage{caption}
\usepackage{subcaption}
\usepackage{bm}
\usepackage{isomath}

\usepackage[numbers]{natbib}

\usepackage[pagebackref=true,breaklinks=true,letterpaper=true,colorlinks,bookmarks=false]{hyperref}
\usepackage{footmisc}
\usepackage{stmaryrd}

\usepackage{fixltx2e}
\usepackage{dblfloatfix}
\usepackage{pbox}
\usepackage{capt-of}

\usepackage[normalem]{ulem}
\usepackage{multirow}

\usepackage{colortbl}

\usepackage{mathtools}
\usepackage{array}

\usepackage{amsmath}
\usepackage{amssymb}
\usepackage{amsthm}
\usepackage{indentfirst}

\makeatletter
\@namedef{ver@everyshi.sty}{}
\makeatother
\usepackage{pgf}

\input{definitions.tex}
\newcommand{\comment}[1]{}
\newcommand*{\barbar}[1]{\bar{\bar{#1}}}

\usepackage{tikz}

\wacvfinalcopy % *** Uncomment this line for the final submission

\def\wacvPaperID{404} % *** Enter the wacv Paper ID here

\ifwacvfinal\fi
\begin{document}

%%%%%%%%% TITLE
\title{Power Normalizing Second-order Similarity Network for Few-shot Learning}

% Authors at the same institution
%\author{First Author \hspace{2cm} Second Author \\
% Authors at different institutions
\author{Hongguang Zhang \\
Australian National University, Data61/CSIRO\\
{\tt\small hongguang.zhang@anu.edu.au}
\and
Piotr Koniusz \\
Australian National University, Data61/CSIRO\\
{\tt\small piotr.koniusz@\{anu.edu.au, data61.csiro.au\}}
}

\maketitle
\ifwacvfinal\thispagestyle{empty}\fi

%%%%%%%%% ABSTRACT
\input{abstract.tex}

%%%%%%%%% BODY TEXT
\input{introduction.tex}
\input{related.tex}

\input{background.tex}

\input{approach.tex}

\input{experiments.tex}
\input{conclusions.tex}
\input{acknowledgement.tex}
\renewcommand*\appendixpagename{Appendix}
\begin{appendices}

\input{appendix.tex}
\end{appendices}

%-------------------------------------------------------------------------

{\small
\bibliographystyle{ieee}
\bibliography{fsl}
}

\end{document}

%% file: definitions.tex
\renewcommand\vec[1]{\ensuremath\boldsymbol{#1}}
\renewcommand\cdots{...}

\newcommand{\tF}{\vec{\mathcal{F}}}

\newcommand{\mY}{\mathbf{Y}}

\newcommand{\tX}{\vec{\mathcal{X}}}

\newcommand{\mX}{\mathbf{X}}
\newcommand{\vx}{\mathbf{x}}

\newcommand{\mbr}[1]{\mathbb{R}^{#1}}

\newcommand{\idx}[1]{\mathcal{I}_{#1}}

\newcommand{\vphi}{\boldsymbol{\phi}}
\newcommand{\vpsi}{\boldsymbol{\psi}}

\newcommand{\mPsi}{\vec{\Psi}}

\newcommand{\set}[1]{\left\{#1\right\}}

\DeclareMathOperator*{\argmin}{arg\,min}

\DeclareMathOperator*{\trace}{Tr}
\DeclareMathOperator*{\kronstack}{\uparrow\!\otimes}

\DeclareMathOperator*{\avg}{avg}

\newcommand{\expl}[1]{\text{e}^{#1}}

\newcommand{\suptensorr}[2]{\mathfrak{S}^{#1}_{\times^{#2}}}

\newtheorem{proposition}{Proposition}
\newtheorem{remark}{Remark}

\def\eg{\emph{e.g.}}

\newcommand{\mygthree}[1]{\boldsymbol{\mathcal{G}}\!\left(\!#1\!\right)}

\newcommand{\tG}{\boldsymbol{\mathcal{G}}}

\newcommand{\mIdent}{\boldsymbol{\mathds{I}}}

\newcommand{\vOnes}{\mathbb{1}}

\newcommand{\tS}{\vec{\mathcal{S}}}

\newcommand{\tNnb}{\mathcal{N}}

\newcommand{\mPi}{\boldsymbol{\Pi}}

\newcommand{\mPhi}{\boldsymbol{\Phi}}

\newcommand{\mM}{\boldsymbol{M}}

\newcommand{\vmu}{\boldsymbol{\mu}}

\newcommand{\stkout}[1]{{\ifmmode\text{\sout{\ensuremath{#1}}}\else\sout{#1}\fi}}

\DeclareMathOperator*{\arcsinh}{arcsinh}

%% file: abstract.tex
\vspace{-0.2cm}
\begin{abstract}
\vspace{-0.2cm}
Second- and higher-order statistics of data points have played an important role in advancing the state of the art on several computer vision problems such as the fine-grained image and scene recognition. However, these statistics need to be passed via an appropriate pooling scheme to obtain the best performance. Power Normalizations are non-linear activation units which enjoy probability-inspired derivations and can be applied in CNNs. In this paper, we propose a similarity learning network leveraging second-order information and Power Normalizations. To this end, we propose several formulations capturing second-order statistics and derive a sigmoid-like Power Normalizing function to demonstrate its interpretability. Our model is trained end-to-end to learn the similarity between the support set and query images for the problem of one- and few-shot learning. The evaluations on Omniglot, \textit{mini}Imagenet and Open MIC datasets demonstrate that this network obtains state-of-the-art results on several few-shot learning protocols.
%\dots
\end{abstract}

%% file: introduction.tex
\vspace{-0.15cm}
\section{Introduction}
\label{sec:intro}
CNNs have improved the classification results on numerous tasks such as the object category recognition,  scene classification and fine-grained image recognition. However, they rely on a large number of labeled images for training and the training procedure itself is a time consuming task. 

In contrast, people's ability to learn and recognize new objects and even complex visual concepts from very few samples highlights the superiority of biological vision over artificial CNNs. Inspired by the human ability of learning in the few-samples regime, researchers have turned their attention to the so-called problem of few-shot learning for which networks are trained by the use of only few training samples. Recently, simple relation-learning deep networks have been proposed \cite{vinyals2016matching,snell2017prototypical,sung2017learning,NIPS2017_7082} which can be viewed as a form of metric learning \cite{metric_old,kissme,Mehrtash_CVPR_2018} adapted to the few-shot learning scenario. We take a similar view on the one- and few-shot learning problem, however, we will focus on capturing robust statistics for the purpose of similarity learning.

Second-order statistics of data features have played a pivotal role in advancing the state of the art on several problems in computer vision, including object recognition, texture categorization, action representation, and human tracking, to name a few of applications \cite{tuzel_rc,porikli2006tracker,wang2011tracking,guo2013action,carreira_secord,me_tensor_tech_rep,koniusz2017higher}. For example, in the popular region covariance descriptors~\cite{tuzel_rc}, a covariance matrix, which is computed over multi-modal features from image regions, is used as an object representation for recognition and tracking, and has been extended to several other applications~\cite{tuzel_rc, porikli2006tracker,wang2011tracking,guo2013action}. Recently, second-order representations have been extended to CNNs for end-to-end learning and they obtained state-of-art results on action recognition, texture classification, scene and fine-grained recognition tasks \cite{koniusz2016tensor,koniusz2017domain,koniusz2018deeper}. 
%The ability of such representations to capture correlation patterns via co-occurrence and/or higher-order occurrence statistics of features of datapoints has attracted visible attention from the community \cite{bilinear_finegrained,face_cooc,deep_cooc,bilinear_finegrained}. 

In this paper, we  design a similarity learning network by leveraging second-order statistics and Power Normalization. In contrast to the large-scale object recognition, one- and few-shot learning problems warrant an investigation into these representations to determine the best performing architecture. As second-order statistics require appropriate pooling for such a new task, we also investigate pooling via Power Normalizations. It is known that such operators act as detectors of visual words and thus they discard nuisance variability related to the frequency of certain visual words  varying from image to image of the same class \cite{koniusz2018deeper}. We speculate that, as we capture relationships between multiple images in a so-called episode to learn similarity between images, such nuisance variability would be multiplicative in its nature in the number of images per episode. Power Normalizations can limit such a harmful effect.

For contributions, we (i) investigate how to build second-order representations from feature maps of the last fully-convolutional layer of CNN to devise relationship descriptors, (ii) we propose to permute second-order matrices to capture local correlations, (iii) we derive an interpretable non-linearity for pooling (our Appendix) from the Power Normalization family \cite{koniusz2018deeper}, (iv) we propose a new one-shot learn. protocol for the challenging Open MIC dataset \cite{me_museum}.

We are the first to embed second-order statistics and Power Normalization into one- and few-shot learning. The relationship descriptors passed via such a non-linearity are then passed to a so-called similarity learning network.

%To the best of our knowledge, 
%We are the first to embed second-order statistics and Power Normalization into a one- and few-shot learning setting. %Our evaluations consider possible representations and lead to the state of the art  in one- and few-shot learning. % on three publicly available benchmarks.

%% file: related.tex
\section{Related Work}
\label{sec:related}

In what follows, we describe popular zero-, one- and few-shot learning algorithms followed by a short discussion on second-order statistics and Power Normalization.

\subsection{Learning From Few Samples}
\label{sec:related_few_shot}

For deep learning algorithms, the ability of {\em``learning quickly from only a few examples is definitely the desired characteristic to emulate in any brain-like system''} \cite{book_nip}. This desired operating principle poses a challenge to typical CNN-based big data approaches \cite{ILSVRC15} designed for large scale visual category learning. Current trends in computer vision highlight the need for {\em ``an ability of a system to recognize and apply knowledge and skills learned in previous tasks to novel tasks or new domains, which share some commonality''}. This problem was introduced in 1901 under a notion of ``{\em transfer of particle}''~\cite{woodworth_particle} and is closely related to zero-shot learning \cite{larochelle2008zero,farhadi2009describing,akata2013label} which can be defined as an ability to generalize to unseen class categories from categories seen during training. For one- and few-shot learning, some ``{\em transfer of particle}'' is also a desired mechanism as generalizing from one or few datapoints to account for intra-class variability of thousands images is a formidable task.

\vspace{0.05cm}
\noindent{\textbf{One- and Few-shot Learning }} has been  studied widely  in computer vision in both shallow \cite{miller_one_example,Li9596,NIPS2004_2576,BartU05,fei2006one,lake_oneshot} and deep learning scenarios \cite{koch2015siamese,vinyals2016matching,snell2017prototypical,finn2017model,snell2017prototypical,sung2017learning}. 

Early works \cite{fei2006one,lake_oneshot} propose one-shot learning methods motivated by the observation that humans can learn new concepts from very few samples. These two papers employ a generative model with an iterative inference for transfer. %in order to take advantage of knowledge from previously learned tasks. 

Siamese Network \cite{koch2015siamese} presents a two-streams convolutional neural network approach which generates image descriptors and learns the similarity between them. %The learning protocol learns to distinguish whether two datapoints come from the same or different class. %Therefore, one can see such an approach as similarity learning which is related to metric learning. 
Matching Network \cite{vinyals2016matching} introduces the concept of support set and $L$-way $Z$-shot learning protocols. It  captures the similarity between one testing and several support images, thus casting the one-shot learning problem as set-to-set learning. Such a network work with unobserved classes without any modifications. 
Prototypical Networks \cite{snell2017prototypical} learns a model that computes distances between a datapoint and prototype representations of each class. %
Model-Agnostic Meta-Learning (MAML) \cite{finn2017model} %introduces a meta-learning model which %can be considered a form of transfer learning. %as otherwise fine-tuning on few datapoints per classes is challenging.
%Their model 
is trained on a variety of different learning tasks. %This results in a good initial condition for the solver to generalize to other novel tasks. % with few training samples. 
Lastly, Relation Net \cite{sung2017learning} %proposes a simple but
is an effective end-to-end network for learning the relationship between testing and support images. Conceptually, this model is similar to Matching Network \cite{vinyals2016matching}. However, Relation Net leverages an additional deep neural network to learn similarity on top of the image descriptor generating network. Then, the similarity network generates so-called relation scores. 

Our work is somewhat similar to Relation Net \cite{sung2017learning} in that we use the feature encoding and the similarity networks. However, approach \cite{sung2017learning} uses first-order representations for similarity learning. In contrast, we investigate second-order representations to capture co-occurrences of features. To this end, we also look at Power Normalizing functions which role is to aggregate feat. vectors robustly.

\vspace{0.05cm}
\noindent{\textbf{Zero-shot Learning}} can be implemented within the similarity learning frameworks which follow \cite{koch2015siamese,vinyals2016matching,snell2017prototypical,sung2017learning}. Below we summarize popular zero-shot learning methods.

Attribute Label Embedding (ALE) \cite{akata2013label} uses attribute vectors as label embedding and uses an objective inspired by a structured WSABIE ranking method, which assigns more importance to the top of the ranking list. Embarrassingly Simple Zero-Shot Learning (ESZSL) \cite{romera2015embarrassingly} implements regularization terms for a linear mapping to penalize the projection of feature vectors to the attribute space and the projection of attribute vectors to the feature space. Zero-shot Kernel Learning (ZSKL) \cite{zhang2018zero} proposes a non-linear kernel method which realizes the compatibility function. The weak incoherence constraint is applied in this model to make the columns of projection matrix incoherent. Feature Generating Networks \cite{xian2017feature} uses a generative adversarial network to hallucinate feature vectors for unobserved classes, thus correcting the imbalance in the data distribution. Differently \cite{zhang2018model} proposes to train an extra SVM to detect if samples are from seen classes or unseen classes, thus employing different classifier to solver the problem.% on the so-called generalized zero-shot learning protocol.  

\subsection{Second-order Statistics/Power Normalization}
\label{sec:related_pn}

Below we discuss several shallow and CNN-based methods which use second-order statistics followed by details on so-called pooling and Power Normalization functions.

\vspace{0.05cm}
\noindent{\textbf{Second-order statistics }}
have been studied in the context of texture recognition \cite{tuzel_rc, elbcm_brod} by so-called Region Covariance Descriptors (RCD) %. Such methods use a representation which typically captures co-occurrences of luminance, first- and/or second-order derivatives of texture patterns. RCD approaches 
which have also been  applied to tracking \cite{porikli2006tracker}, semantic segmentation \cite{carreira_secord} and object category recognition \cite{me_tensor_tech_rep,koniusz2017higher}, to name but a few of applications.

Recently, using co-occurrence patterns in CNN setting, similar in spirit to RCD, has become a popular direction. Approach \cite{bilinear_finegrained} fuses two CNN streams via outer product in the context of fine-grained image recognition. %Face recognition algorithm \cite{face_cooc} uses co-occurrences of CNN feature vectors and facial attribute vectors to obtain state-of-the-art face recognition results. 
A recent approach \cite{deep_cooc} extracts feature vectors at two separate locations in feature maps and performs an outer product to form a CNN co-occurrence layer. Higher-order statistics have also been used for action classification from the body skeleton sequences \cite{koniusz2016tensor} and for domain adaptation \cite{koniusz2017domain}.

Our work differs from the above methods in that we perform end-to-end training for one- and few-shot learning via relation learning. We devise a relation descriptor that captures several support and query images before being passed to the similarity learning network extracting relationship. % for the support and query images. %from the descriptor statistics.

\vspace{0.05cm}
\noindent{\textbf{Power Normalizations. }}
Second-order statistics have to deal with the so-called burstiness which is ``{\em the property that a given visual element appears more times in an image than a statistically independent model would predict}'' \cite{jegou_bursts}. Power Normalization~\cite{me_ATN,me_tensor_tech_rep} is known to suppress this burstiness %, it is related to the problem of robust estimation of statistics, 
and has been extensively studied and evaluated in the context of Bag-of-Words \cite{me_ATN,me_tensor_tech_rep,koniusz2017higher,koniusz2018deeper}. 

A theoretical relation between Average and Max-pooling was studied in~\cite{boureau_midlevel}. %which highlighted the underlying statistical reasons for the superior performance of Max-pooling. %compared to a mere average of feature vectors. 
An analysis of pooling was conducted in~\cite{boureau_pooling} under  assumptions on distributions from which %the aggregated 
features are drawn.  %A relationship between the likelihood of `\emph{at least one particular visual word being present in an image}' and Max-pooling was studied in \cite{liu_sadefense}. 
Later, a survey \cite{me_ATN} showed that so-called MaxExp feature pooling in~\cite{boureau_pooling} can be interpreted as a detector of ``\emph{at least one particular visual word being present in an image}''. According to the survey \cite{me_ATN}, many Power Normalization functions are closely related. This view was further extended to pooling on second-order matrices in  %the Arcsin hyperbolic function and the Logistic a.k.a. Sigmoid function (SigmE).
\cite{koniusz2018deeper}. %Their functions are statistically derived specifically for second-order representations rather than first-order feature vectors.%include spectral Power Normalization family which is costly due to the need to backpropagate via SVD of second-order matrices. Thus, spectral functions are left out from our experiments for the future work.

Our approach leverages Power Normalization functions for second-order statistics. As a theoretical contribution, we present an extended derivation of pooling which operates on auto-correlation matrices that typically contain positive and negative evidence of feature co-occurrences. In contrast, Power Normalization functions such as the square root and MaxExp operate only on non-negative feature values \cite{me_ATN}.

%% file: background.tex
\section{Background}
Below we detail our notations, explain how to compute second-order statistics and %we define basic 
Power Normalization functions. % useful in the sequel.

\subsection{Notations}
Let $\vx\in\mbr{d}$ be a $d$-dimensional feature vector. $\idx{N}$ stands for the index set $\set{1, 2,\cdots,N}$. Let $\vx\in\mbr{d}$ be a $d$-dimensional feature vector. Then we use $\tX\!=\!{\kronstack}_r\,\vx$ to denote the $r$-mode super-symmetric rank-one tensor $\tX$ generated by the $r$-th order outer-product of $\vx$, where the element of $\tX\!\in\!\suptensorr{d}{r}$ at the $\left(i_1,i_2,\cdots, i_{r}\right)$-th index is given by $\Pi_{j=1}^r x_{i_j}$.  
%
%The Frobenius norm of matrix is given by  $\fnorm{\mX}\!\!=\!\!\!\sqrt{\sum\limits_{m,n} \!\!X_{mn}^2}$, where $X_{mn}$ represents the $\left(m,n\right)$-th element of $\mX$.  
%The spaces of symmetric positive semidefinite and definite matrices are $\semipd{d}$ and $\spd{d}$. 
%A vector with all coefficients equal one is denoted by $\vOnes$. and $\mJ_{mn}$ is a matrix of all zeros with one in at position $(m,n)$.
%The MATLAB-style operators for matrix vectorization and matrix reshaping to the size of $(m,n)$ are denoted by $(:)$, \ie, $\mX_{(:)}$, and $\res(\mX,m,n)$.
%
%Operator $\left[k_{ij}\right]_{i,j\!\in\!\idx{N}}$ denotes stacking coefficients $k_{ij}$ into matrix $\mK$ of size $N\!\times\!N$. 
%Moreover, $\delta(x)\!=\!\lim_{\sigma\rightarrow 0}\exp(-x^2/(2\sigma^2))$ returns one for $x\!=\!0$ and zero for $x\!\neq\!0$. 
We also define $\vOnes\!=\![1,\cdots,1]^T$. Operators $;_1$, $;_2$ and $;_3$ denote concatenation along the first, second and third mode, respectively. Operator $(:)$ denotes vectorisation of a matrix or tensor. Typically, capitalised boldface symbols such as $\mPhi$ denote matrices, lowercase boldface symbols such as $\vphi$ denote vectors and regular case such as $\Phi_{ij}$, $\phi_{i}$, $n$ or $Z$ denote scalars, \eg~$\Phi_{ij}$ is the $(i,j)$-th coefficient of $\mPhi$.

%------------------------------------------------------------------------- 
\subsection{Second- and High-order Tensors}
\label{sec:som}
Below we show that second- or higher-order tensors emerge from a linearization of sum of Polynomial kernels.
\begin{proposition}
\label{pr:linearize}
Let $\mPhi_A\equiv\{\vphi_n\}_{n\in\tNnb_{\!A}}$, $\mPhi_B\!\equiv\{\vphi^*_n\}_{n\in\tNnb_{\!B}}$ be datapoints from two images $\Pi_A$ and $\Pi_B$, and $N\!=\!|\tNnb_{\!A}|$ and $N^*\!\!=\!|\tNnb_{\!B}|$ be the numbers of data vectors \eg, obtained from the last convolutional feat. map of CNN for images $\Pi_A$ and $\Pi_B$. Tensor feature maps result from a linearization of the sum of Polynomial kernels of degree $r$:
\vspace{-0.1cm}
\begin{align}
& K(\mPhi_A, \mPhi_B)\!=\!\left<\mPsi(\mPhi_A),\mPsi(\mPhi_B)\right>\!=\label{eq:hok1}\\
& \!\!\!\!\frac{1}{NN^*\!}\!\!\sum\limits_{
n\in \tNnb_{\!A}}\sum\limits_{n'\!\in \tNnb_{\!B}\!}\!\!\!\left<\vphi_n, \vphi^*_{n'}\right>^r\!\text{ where } 
\mPsi(\mPhi)\!=\!\frac{1}{N}\sum\limits_{
n\in \tNnb}{\kronstack}_r\,\vphi_n.\!\!\nonumber
\end{align}
\end{proposition}
\begin{proof}
See \cite{me_tensor_tech_rep,koniusz2017higher} for a derivation of this expansion.
\end{proof}

\begin{remark}
\label{re:sec_ord_model}
In what follows, we will use second-order matrices obtained from the above expansion for $r\!=\!2$ built from datapoints $\vphi$ and $\vphi^*\!$ which are partially shifted by their means $\vmu$ and $\vmu^*\!$ so that $\vphi\!:=\!\vphi\!-\!\beta\vmu$, $\vphi^*\!:=\!\vphi^*\!-\!\beta\vmu^*\!$, and $0\!\leq\!\beta\!\leq\!1$ to account for so-called negative visual words which are the evidence of
lack of a given visual stimulus in an image \cite{koniusz2018deeper}. %Thus, we obtain:
\comment{
\begin{align}
& \!\!\!\!\!\!\!\frac{1}{NN^*\!}\!\!\sum\limits_{
n\in \tNnb_{\!A}}\sum\limits_{n'\!\in \tNnb_{\!B}}\!\!\!\left<\vphi_n, \vphi^*_{n'}\!\right>^2\!\!=\!
\Big\langle\frac{1}{N}\sum\limits_{
n\in \tNnb_{\!A}}{\vphi_n\vphi_n^T}, \frac{1}{N^*\!}\sum\limits_{
n\in \tNnb_{\!B}}{\vphi^*_{n'}{\vphi^*_{n'}}^{\!\!\!T}}\Big\rangle.\!\!\label{eq:hok2}
\end{align}
}
\footnotetext[3]{\label{foot:maps}Note that (kernel) feature maps are not conv. CNN maps.  %These are two different notions.
}
We introduce a pooling operator $\mygthree{\,\mX}\!=\!\mX$ which will be replaced by various Power Normalization functions. This yields a (kernel) feature map{\color{red}\footnotemark[3]}:
\begin{align}
& \!\!\!\mPsi\left(\{\vphi_n\}_{n\in\tNnb}\right)=\tG\Big(\frac{1}{N}\sum_{n\in\mathcal{N}}\!\vphi_n\vphi_n^T\Big)=\tG\Big(\frac{1}{N}\mPhi\mPhi^T\Big).\label{eq:hok3}
\end{align}
\vspace{-0.5cm}
\end{remark}

%------------------------------------------------------------------------- 
\subsection{Power Normalization Family}

Max-pooling \cite{boureau_midlevel} was derived by drawing features from the Bernoulli distribution under the i.i.d. assumption \cite{boureau_pooling} which leads to so-called {\em Theoretical Expectation of Max-pooling} ({\em MaxExp}) operator \cite{me_ATN}  extended to work with non-negative feature co-occurrences \cite{koniusz2018deeper} as detailed below.

\begin{proposition}
\label{pr:cooc}
Assume two event vectors $\vphi,\vphi'\!\!\in\!\{0,1\}^{N}$ which store the $N$ trials each, performed according to the Bernoulli distribution under i.i.d. assumption, 
 for which the probability $p$ of an event $(\phi_{n}\!\cap\!\phi'_{n}\!=\!1)$ denotes a co-occurrence and $1\!-\!p$, for $(\phi_{n}\!\cap\!\phi'_{n}\!=\!0)$, denotes the lack of it, and $p$ is estimated as an expected value $p\!=\!\avg_n\phi_n\phi'_{n}$. Then the probability of at least one co-occurrence event $(\phi_{n}\!\cap\!\phi'_{n}\!=\!1)$ in $\phi_n$ and $\phi'_n$ simultaneously in $N$ trials becomes $\psi\!=\!1\!-\!(1\!-\!p)^{N}$.
\end{proposition}
\begin{proof}
See \cite{koniusz2018deeper} for a detailed proof.
\end{proof}
\begin{remark} 
%\vspace{-0.3cm}
\label{re:maxexp}
A practical variant of this pooling method \cite{koniusz2018deeper} is given by $\psi_{kk'}\!=\!1\!-\!(1\!-\!\avg_n\phi_{kn}\phi_{k'n})^{\eta}$, where $0\!<\!\eta\!\approx\!N$ is an adjustable parameter, $\phi_{kn}$ and $\phi_{k'n}$ are $k$-th and  $k'\!$-th features of an $n$-th feature vector $0\!\leq\!\vphi\!\leq\!1$ \eg, as defined in Prop. \ref{pr:linearize}, which is normalized to range 0--1.
\end{remark}
\begin{remark}
\label{re:pn}
It was shown in \cite{koniusz2018deeper} that Power Normalization ({Gamma}) given by $\psi_{kk'}\!=\!(\avg_n\phi_{kn}\phi_{k'n})^\gamma$, for $0\!<\!\gamma\!\leq\!1$ being a parameter, is in fact an approximation of MaxExp. In matrix form, if $\mM\!=\!\frac{1}{N}\mPhi\mPhi^T\!$, matrix $0\!\leq\!\mPhi\!\leq\!1$ contains datapoints $\vphi_1,\cdots,\vphi_N$ as column vectors, then $\mPsi\!=\!\mygthree{\,\mM,\gamma\,}\!=\mM^\gamma$ (element-wise rising to the power of  $\gamma$). % is an element-wise operation.
\end{remark}

\begin{figure*}[t]
\vspace{-0.3cm}
\centering
\includegraphics[height=5cm]{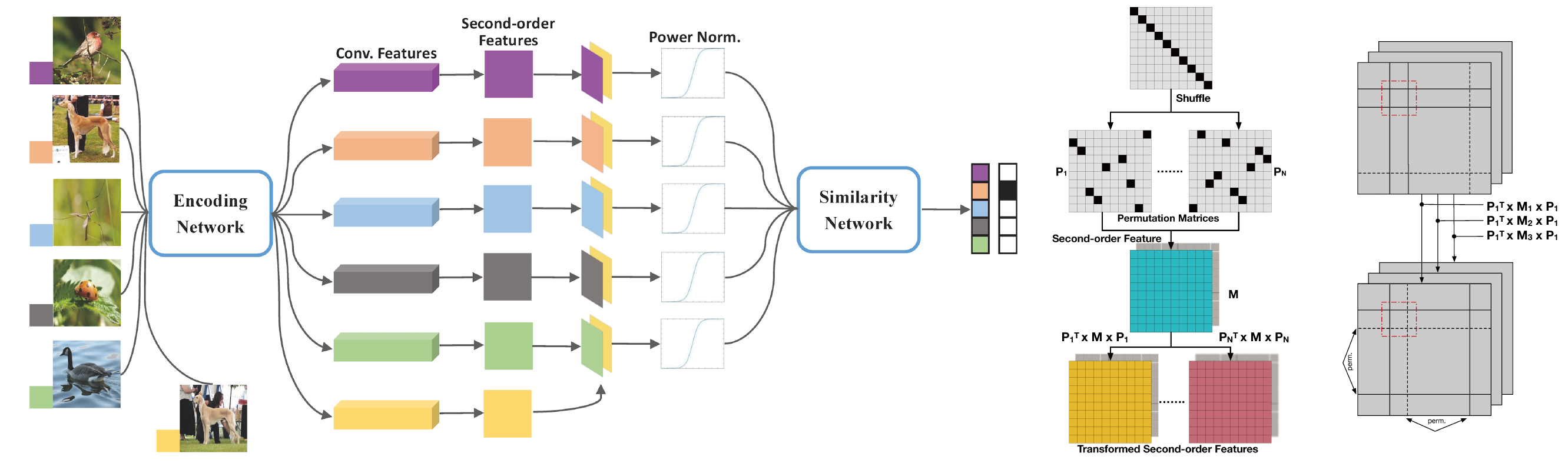}
\\\vspace{-0.4cm}
\begin{subfigure}[t]{0.50\linewidth}
\caption{\label{fig:pipe1}}
\end{subfigure}
\begin{subfigure}[t]{0.3\linewidth}
\caption{\label{fig:pipe2}}
\end{subfigure}
\begin{subfigure}[t]{0.01\linewidth}
\caption{\label{fig:pipe3}}
\end{subfigure}
%{
%\fontsize{8.5}{9}\selectfont
%\begin{tabular}{c c c}
%\hspace{3cm}\textbf{(a)} & \hspace{8cm}\textbf{(b)} & %\hspace{1.2cm}\textbf{(c)} 
%\end{tabular}
%}
\vspace{-0.4cm}
\caption{Our second-order similarity network. In Figure \ref{fig:pipe1}, the outer-product is applied on the convolutional features to obtain the second-order feature matrix, which is followed by a Power Normalization layer. Then the normalized second-order feature matrix is fed into the similarity network to learn the relation between support and query images. In Figure \ref{fig:pipe2}, we show that we can use permutations of second-order matrix so that local filters (red dashed square in Figure \ref{fig:pipe3}) of similarity network can capture various correlations.}
\label{fig:sosn}
\vspace{-0.3cm}
\end{figure*}

\begin{remark}
\label{re:asinhe}
It was shown in \cite{koniusz2018deeper} that AsinhE function is an extension of Gamma which acts on $\vphi\!\in\!\mbr{}$ rather than $\vphi\!\geq\!0$. It is defined as the Arcsin hyperbolic function:
\begin{align}
& \!\!\mPsi\!=\!\mygthree{\,\mM,\eta\,}\!=\arcsinh(\gamma'\!\mM)\!=\!\log(\gamma'\!\mM+\sqrt{1+{\gamma'}^2\!\mM^2}),
\end{align}
where $\mM\!=\!\frac{1}{N}\mPhi\mPhi^T\!$,  $\mPhi$ contains datapoints $\vphi_1,\cdots,\vphi_N$ as column vectors, param. $\gamma'$ corresponds to $\gamma$ in Remark \ref{re:pn}.
\end{remark}

\begin{remark}
\label{re:pnpn}
It was also shown in \cite{koniusz2018deeper} that the MaxExp operator can be extended to act on $\vphi\!\in\!\mbr{}$ rather than $\vphi\!\geq\!0$ by the use of Logistic a.k.a. Sigmoid ({\em SigmE}) functions:
\begin{align}
& \!\!\!\!\!\mPsi\!=\!\mygthree{\,\mM,\eta'\,}\!=\!\frac{2}{1\!+\!\expl{-\eta'\mM}}\!-\!1\text{ and }\frac{2}{1\!+\!\expl{\frac{-\eta'\mM}{\trace(\mM)+\lambda}}}\!-\!1,\!%\quad\text{and}\quad\mM\!=\!\frac{1}{N}\sum\limits_{n\in\mathcal{N}}\vphibar_n\vphibar_n^T,
\label{eq:sigmoid}
\end{align}
where $\mM$ is def. as in Remark \ref{re:asinhe}, parameter $\eta'$ corresponds to $\eta$ in Remark \ref{re:maxexp}, $\lambda\!\approx\!1e\!-\!6$ is a small regularization constant and $\trace(\mM)$ prevents input values from  exceeding one.
\end{remark}

In Section \ref{sec:pool_der}, we show that the SigmE operator indeed is a close approximation of our extended MaxExp$\,$($\pm$) pooling, which we derive in Section \ref{sec:pool_der}. It acts on $\vphi\!\in\!\mbr{}$ rather than $\vphi\!\geq\!0$ as our model in Remark \ref{re:sec_ord_model} may contain negative co-occurrences. In our experiments, we use AsinhE/SigmE rather than Gamma/MaxExp for the above reason. %Below, we describe our approach in  detail.

%% file: approach.tex
\section{Approach}
We start by describing the components of network %used in our experiments followed by description of 
followed by our relationship descriptor/operator which role is to capture co-occurrences in image representations. Next, we propose a permutation-based representation for second-order matrices. In Appendix, we derive  MaxExp$\,$($\pm$) pooling  which is closely approximated by SigmE pooling and offers an interpretation for such  operators.

\subsection{Network}

Inspired by the end-to-end relationship-learning network \cite{sung2017learning}, our Second-order Similarity Network (SoSN), shown in Figure \ref{fig:sosn}, consists of two major parts which are (i) feature encoding network and (ii) similarity network. The role of the feature encoding network is to generate convolutional feature vectors which are then used as image descriptors. The role of the similarity network is to learn the relation and compare so-called support and query image embeddings. Our work is different to the Relation Net \cite{sung2017learning} in that we apply and study several variants of second-order representations built from image descriptors followed by two variants of Power Normalizing functions. For instance, we construct the support and query second-order feature matrices followed by a non-linear Power Normalization unit. In SoSN, the feature encoding network remains the same as Relation Net \cite{sung2017learning}, however, the similarity network learns to compare from second- rather than the first-order statistics. %
The architecture of our network is given in Figures \ref{fig:sosn} and \ref{fig:blocks}.
%The similarity network contains four convolutional blocks and two fully-connected layers. Each convolutional block is the combination of a convolution with filters, Batch Normalization, ReLU activation and the max-pooling layer.

Let us define the feature encoding network as an operator $f\!:(\mbr{W\!\times\!H}; \mbr{|\tF|})\!\shortrightarrow\!\mbr{K\!\times\!N}$, where $W$ and $H$ denote the width and height of an input image, $K$ is the length of feature vectors (number of filters), $N\!=\!N_W\!\cdot\!N_H$ is the total number of spatial locations in the last convolutional feature map. For simplicity, we denote an image descriptor by $\mPhi\!\in\!\mbr{K\!\times\!N}$, where $\mPhi\!=\!f(\mX; \tF)$ for an image $\mX\!\in\!\mbr{W\!\times\!H}$ and $\tF$ are the parameters-to-learn of the encoding network.

The similarity network, which role is to compare two datapoints encoded as some $K'\!$ dim. vectorized second-order representations, is denoted by $s\!:(\mbr{K'\!}; \mbr{|\tS|})\!\shortrightarrow\!\mbr{}$. Typically, we write $s(\vpsi; \tS)$, where $\vpsi\!\in\!\mbr{K'}\!$ and $\tS$ are the parameters-to-learn of the similarity network.
%Note that $g$ will take a par consisting of the support descriptor $\vv\!\in\!\mbr{K'\!}$ and a query descriptor $\vq\!\in\!\mbr{K'\!}$ to score their similarity, \eg~$g(\vv,\vq)$.

%Variable $L'\!$ is the size of our second-order representation, $W$ denotes the Z-shot problem in which $W$ is the number of support images per class in the episode, and $N\!$ is the number of selected classes from all $C$ classes for the N-way problem such that $\{c_1,\cdots,c_N\}\!\subset\!\{1,\cdots,C\}$. Note that $g$ will take a par consisting of $W$ support image descriptors and one query descriptor.

Moreover, we define a descriptor/operator $\vartheta\!:(\left\{\mbr{K\!\times\!N}\!\right\}^{Z}, \mbr{K\!\times\!N})\!\shortrightarrow\!K'\!$ which encodes some relationship between the descriptors built from the $Z$-shot support images and a query image. This relationship is encoded via computing second-order statistics of descriptors, performing some Power Normalizing function, applying \eg~concatenation, inner-product operation, subtraction, \etc.; all of these operations combined together can capture some form of relationship between features of images.  %$Tr(M)$ is the trace of matrix $M$, $\Psi$ is the power normalization function. 

For the $L$-way $1$-shot problem, we assume one support image $\mX$ with its image descriptor $\mPhi$ and one query image $\mX^*\!$ with its image descriptor $\mPhi^*$. In general, we use `$^*\!$' to indicate query-related variables. Moreover, each of the above descriptors belong to one of $L$ classes in the subset $\{c_1,\cdots,c_L\}\!\subset\!\idx{C}$ that forms so-called $L$-way learning problem and the class subset $\{c_1,\cdots,c_L\}$ is chosen at random from $\idx{C}\!\equiv\!\{1,\cdots,C\}$. Then, the $L$-way $1$-shot learning step can be defined as learning similarity:
\begin{equation}
SI(\mPhi,\mPhi^*\!;\tS)=s\left(\vartheta\!\left(\mPhi,\mPhi^*\!\right); \tS\right),\label{eq:simsim}
\end{equation}
where $\tS$ denotes network parameters that have to be learnt.

For the $L$-way $Z$-shot problem, we assume some $W$ support images $\{\mX_n\}_{n\in\mathcal{W}}$ from some set $\mathcal{W}$ and their corresponding image descriptors $\{\mPhi_n\}_{n\in\mathcal{W}}$ which can be considered as a $Z$-shot descriptor if stacked along the third mode. Moreover, we assume one query image $\mX^*\!$ with its image descriptor $\mPhi^*$. Again, both the $Z$-shot and the query descriptors belong to one of $L$ classes in the subset $\mathcal{C}^{\ddag}\!\equiv\!\{c_1,\cdots,c_L\}\!\subset\!\idx{C}\!\equiv\!\mathcal{C}$. Then, the $L$-way $Z$-shot learning step can be defined as learning similarity:
\begin{equation}
SI(\{\mPhi_n\}_{n\in\mathcal{W}},\mPhi^*\!;\tS)=s\left(\vartheta\!\left(\{\mPhi_n\}_{n\in\mathcal{W}},\mPhi^*\!\right),\tS\right).
\end{equation}

Following approach \cite{sung2017learning}, we use the Mean Square Error (MSE) for the objective of our end-to-end SoSN method:
\begin{align}
&\!\!\!\!\!\!\!\!\!\!\!\!\argmin\limits_{\tF, \tS} \sum\limits_{c\in\mathcal{C}^{\ddag}}\sum\limits_{c'\in\mathcal{C}^{\ddag}} \!\left(SI\left(\{\mPhi_n\}_{n\in\mathcal{W}_c},\mPhi^*_{q\in\mathcal{Q}: \ell(q)=c'},\tS\right)\! - \delta\!\left(c\!-\!c'\right)\right)^2\!\!\!,\nonumber\\
&\qquad\text{ s.t. } \mPhi_n\!=\!f(\mX_n; \tF) \text{ and } \mPhi^*_q\!=\!f(\mX^*_q\!; \tF).
\end{align}
In the above eq., $\mathcal{W}_c$ is a randomly chosen set of $W$ support image descriptors of class $c\!\in\!\mathcal{C}^{\ddag}$, $\mathcal{Q}$ is a randomly chosen set of $L$ query image descriptors so that its consecutive elements belong to the consecutive classes in $\mathcal{C}^{\ddag}\!\equiv\!\{c_1,\cdots,c_L\}$. Lastly, $\ell(q)$ corresponds to the label of $q\!\in\!\mathcal{Q}$.

\subsection{Relationship Descriptor/Operator $\vartheta$}

Below we offer some choices for the operator $\vartheta$ which role is to capture/summarize the information held in the support/query image representations. Such a summary is then passed to the similarity network $s$ for learning comparison.

Table \ref{tab_methods} details variants of the relationship operator $\vpsi=\vartheta\!\left(\{\mPhi_n\}_{n\in\mathcal{W}},\mPhi^*\!\right)$. % which we use in our experiments.
\begin{table*}[t]
\vspace{-0.3cm}
\centering
%\makebox[\textwidth]{
{
\fontsize{8.5}{9}\selectfont
\begin{tabular}[t]{c|c|c|c}
%\hline
%Operator & $\vpsi\!=\!\vartheta\!\left(\{\mPhi_n\}_{n\in\mathcal{W}},\mPhi^*\!\right)$\\
%
Operator &\begin{minipage}{3.5cm}Auto-correlation Full ($\otimes$+F)\end{minipage} & \begin{minipage}{4.5cm}Auto-corr., rank. diff., concat. ($\otimes$+R)\end{minipage} & \begin{minipage}{4.5cm}Auto-corr., aver. shots $\bar{\mPhi}$, concat. ($\otimes$)\end{minipage} \\
\hline
$\vpsi\!=\!\vartheta\!\left(\cdots\right)$
&{\begin{minipage}{4.5cm}
\begin{equation}
\hspace{-1.6cm}
\left[\tG\left(\frac{1}{N}\barbar{\mPhi}\barbar{\mPhi}^T\right)\right]_{(:)}\!\!\!\!\text{where}\label{eq:concat_f}
\end{equation}
\vspace{-0.2cm}
\begin{equation}\barbar{\mPhi}\!=\!\big[\left[\mPhi_1;_1\mPhi^*\!\right];_2\cdots;_2\left[\mPhi_W;_1\mPhi^*\!\right]\big]\nonumber
\end{equation}
\end{minipage}
} &
{\begin{minipage}{5.1cm}
\vspace{-0.5cm}
\begin{equation}
\hspace{-1.2cm}
\left[\tG\left(\frac{1}{NZ}\mPhi\mPhi^T\right)\!;_3 \tG\left(\frac{1}{N}\mPhi^*\!\mPhi^{*T}\right)\right]_{(:)}\!\!\!\!\!\!\!\!
\label{eq:concat_r}
\end{equation}
%\vspace{-0.2cm}
\end{minipage}}
&
{\begin{minipage}{5.0cm}
\vspace{0.2cm}
\begin{equation}
\hspace{-1.8cm}
\left[\tG\left(\frac{1}{N}\bar{\mPhi}\bar{\mPhi}^T\!\right)\!;_3 \tG\!\left(\frac{1}{N}\mPhi^*\!\mPhi^{*T}\right)\right]_{(:)}\!\!\!\!\!\!\!\!\!\!\!\!\!\!\!\!\!\!\!\!\label{eq:concat_best}
\end{equation}
\vspace{-0.4cm}
\begin{equation}
\hspace{-1.1cm}
\text{where }\bar{\mPhi}\!=\!\frac{1}{W}\!\sum_{n\in\mathcal{W}}\!\mPhi_n,\; N\!=\!W\!H\nonumber
\end{equation}
\end{minipage}}\\
\hline
\end{tabular}
}
\caption{List of the relationship descriptors/operators $\vartheta$ we propose and evaluate in our paper.}\label{tab_methods}
\vspace{-0.3cm}
\end{table*}
%
%
%
%
%Operations $\mM\!=\!\mPhi\oplus\mPhi^T\!$ and $\mM\!=\!\mPhi\ominus\mPhi^T\!$  in Eq. \eqref{eq:oplus} and \eqref{eq:ominus} are defined as $M_{ij}\!=\!\sum_n\!\Phi_{in}\!+\!\Phi_{jn}, \forall i,j$ and $M_{ij}\!=\!\sum_n\!\Phi_{in}\!-\!\Phi_{jn}, \forall i,j$. 
%
Operator ($\otimes$+F) concatenates the support and query feature vectors prior to the outer-product. Operator ($\otimes$+R) performs the outer-product on the support and query feature vectors separately prior to concatenation. Operator ($\otimes$) first averages over $Z$ support images (conv. feat. maps) per class followed by the outer-product on the mean support and query vectors and concatenation. %We try operators ($\otimes$--$\odot$) and ($\otimes$--$\ominus$) which compute the Hadamart product and difference between second-order matrices of support and query vectors. Lastly, we capture {\em or-like} and {\em difference} relations between features via ($\otimes$\&$\oplus$) and ($\otimes$\&$\ominus$).
Operator $\tG(\mM)$ performs Power Normalization such as AsinhE or SigmE pooling explained earlier. Lastly, $0\!\leq\!\alpha\!\leq\!1$ balances the contributions of various statistics we concatenate.

%For $N$-way $1$-shot problem, we suppose one support image is $x_c, c\in\{1, 2, ..., N\}$ and the query image is $q_j$, then the similarity between query sample and each support images can be defined as following:

%\begin{equation}
%SI(x_c, q_j) = g(\Psi(Op(\frac{f(x_c)f^T(x_c)}{Tr(f(x_c)f^T(x_c))}, \frac{f(q_j)f^T(g_j)}{Tr(f(q_j)f^T(g_j))})))
%\end{equation}

%Let us consider a more general $N$-way $Z$-shot problem, we define the set of support images from class $c \in \{1, 2, ..., N\}$ as $s_c$, the similarity is written as:
%\begin{equation}
%SI(s_c, q_j) = g(\Psi(Op(\frac{\sum_{i}f(x_i)(\sum_{i}f(x_i))^T}{Tr(\sum_{i}f(x_i)(\sum_{i}f(x_i))^T)}, \frac{f(q_j)f^T(g_j)}{Tr(f(q_j)f^T(g_j))}))), \{x_i|l(x_i) == c\} 
%\end{equation}
%where $l(x)$ means the label of feature $x$, $l(s)$ is the label of feature set $s$.

%We applied Mean Square Error (MSE) as the objective function for the end-to-end SOSN, it can be written as following:
%\begin{equation}
%\argmin\limits_{f, g} \sum\limits_{c}^{N}\sum\limits_{j}^{W} (SI(s_c, q_j) - \vOnes(l(s_c)==l(q_j)))^2
%\end{equation}

\subsection{Multiple co-occurrence matrices}
\label{sec:perm}
As our similarity learning network takes as input co-occurrence matrices stacked along the third-mode, the CNN filters can model only local relations. While there exist spatially local correlations in images (and their feature maps), correlations in matrices are not localized. Thus, some loss of performance is expected. To overcome this, we propose to generate a few of permutations of a second-order matrix that are passed to similarity learning streams. Figure \ref{fig:pipe2} shows this operation. Thus, the CNN filters of each stream can capture and compare different co-occurrence patterns as shown in Figure \ref{fig:pipe3}. Thus, we re-define Eq. \ref{eq:simsim} as follows:
\begin{align}
&SI(\mPhi,\mPhi^*\!;\{\mPi\}_{p\in\idx{P}}, \{\tS\}_{p\in\idx{P}})=\label{eq:simsim2}\\
&\qquad s\left(\left[\mPi_p^T\!\vartheta_q\!\left(\mPhi,\mPhi^*\right)\!\mPi_p\right]_{;_3, \substack{p\in\idx{P}\\q\in\idx{Q}}}; \{\tS\}_{p\in\idx{P}}\right),\nonumber
\end{align}
where $\{\mPi\}_{p\in\idx{P}}$ and $\{\tS\}_{p\in\idx{P}}$ are sets of $P$ permutation matrices ($\mPi_p^T\!\mPi_p\!=\!\mIdent$, $\mPi_p\!\in\!\{0,1\}^{K\!\times\!K}$, $\forall p\!\in\!\idx{P}$, and $\mPi_p\!\neq\!\mPi_q$ if $p\!\neq\!q$) and sim. net. filters. We assume for simplicity that $s(\mPsi, \{\tS\}_{p\in\idx{P}})$ operates on $\mPsi\!\in\!\mbr{K\!\times\!K\!\times\!Q\!\times\!P}$ rather than a vector of size $K'\!\!=\!K^2QP$, $\vartheta_q$ are slices of $\vartheta\!\in\!\mbr{K\!\times\!K\!\times\!Q}$ (not vect.) and $[\cdot]_{;_3}$ is concat. in third mode over matrices. %We use $1\!\leq\!P\!\leq\!10$.

%% file: experiments.tex
\section{Experiments}
\label{sec:exp}
Below we demonstrate usefulness of our proposed relationship descriptors. Our network is evaluated on the Omniglot \cite{lake_oneshot} and \textit{mini}Imagenet \cite{vinyals2016matching} datasets, as well as a recently proposed Open MIC dataset \cite{me_museum}, in the one- and few-shot learning scenario. %All the experiments are built with PyTorch.
%
%In what follows, 
We use images augmented by random rotations of 90, 180 and 270 degrees. We train with Adam solver. The network architecture of our SoSN model is shown in Figure \ref{fig:blocks}.  %in different number of episodes. %
The results are compared against several state-of-the-art methods for one- and few-shot learning.

\subsection{Datasets}
Below, we describe our  setup, datasets and evaluations.

\vspace{0.05cm}
\noindent\textbf{Omniglot } \cite{lake_oneshot} consists of 1623 characters (classes) from 50 alphabets. Samples in each class are drawn by 20 different people. %We follow the standard $L$-way $Z$-shot protocol \cite{sung2017learning}. That is, we pick $L\!=\!5$ or $L\!=\!20$ unseen character classes and provide the model with $Z\!=\!1$ or $Z\!=\!5$ samples. % from each of the $N$ characters as support set $S$ and a batch as query set $B$. 
The dataset is split into 1200 classes for training and 423 classes for testing. All images are resized to $28\!\times\!28$.

\begin{figure}[t]
\vspace{-0.1cm}
	\centering
	\includegraphics[height=2.5cm]{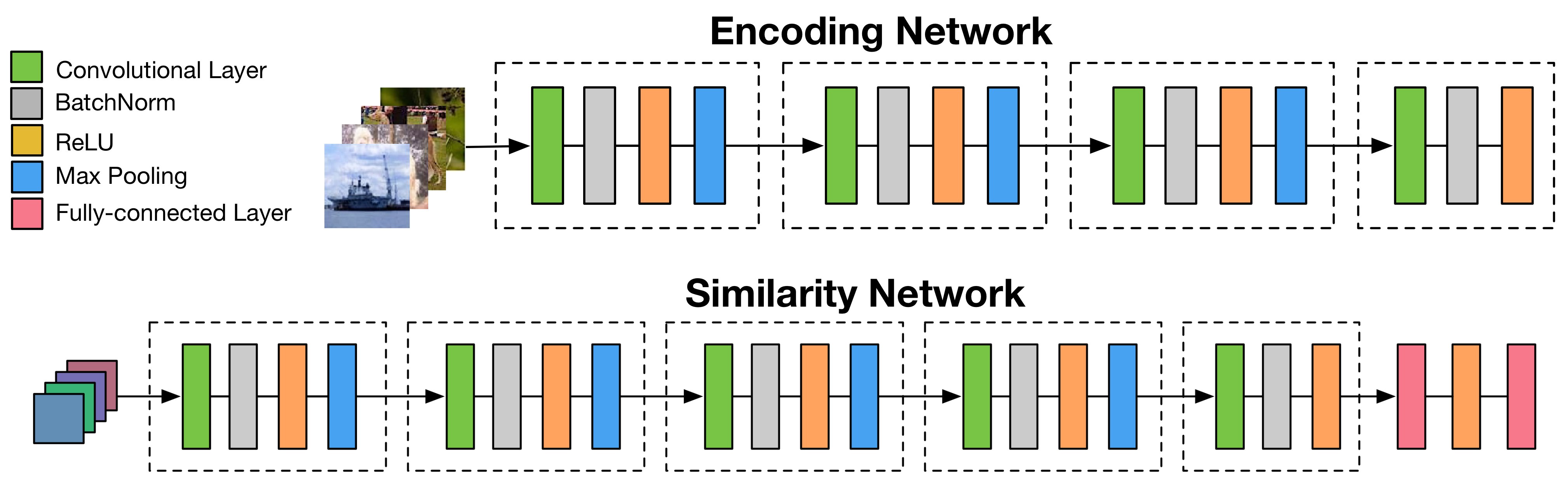}
    \vspace{-0.3cm}
	\caption{The network architecture used in our SoSN model.}
	\label{fig:blocks}
    \vspace{-0.3cm}
\end{figure}

\noindent\textbf{\textit{mini}Imagenet} \cite{vinyals2016matching} consists of 60000 RGB images from 100 classes. %, each class containing 600 samples. 
We follow the standard protocol \cite{vinyals2016matching} and use 64 classes for training, 16 classes for validation and remaining 20 classes for testing, and  use images of size $84\!\times\!84$. %for fair comparison with other methods. 
We also investigate larger sizes, \eg~$224\!\times\!224$, as our SoSN model can use richer spatial information from larger images to obtain high-rank auto-correlation matrices without a need to modify the similarity network to larger feature maps.

In contrast, increasing image size in  Relation Net \cite{sung2017learning} results in a large number of extra convolutional blocks need to to be added to the image descriptor and/or similarity network to deal with the increase in the spatial dimension.

\noindent\textbf{Open MIC}, which stands for the Open Museum Identification Challenge (Open MIC) dataset \cite{me_museum}, is a recently proposed dataset which contains photos of various exhibits, \eg~paintings, timepieces, sculptures, glassware, relics, science exhibits, natural history pieces, ceramics, pottery, tools and indigenous crafts, captured from 10 museum exhibition spaces according to which this dataset is divided into 10 subproblems. In total, it has 866 diverse classes and 1--20 images per class. The within-class images undergo various geometric and photometric distortions as the data was captured with wearable cameras. This makes Open MIC a perfect candidate for testing one-shot learning algorithms. We combine ({\em shn+hon+clv}), ({\em clk+gls+scl}), ({\em sci+nat}) and ({\em shx+rlc}) into subproblems {\em p1}, $\!\cdots$, {\em p4}. Thus, we form 12 possible pairs in which subproblem $x$ is used for training and $y$ for testing (x$\rightarrow$y).
%
%We select the {\em clv}, {\em gls} and {\em nat} subproblems, we form 6 possible pairs in which one subproblem is used for training and another for testing. 
%
%There are 10 museum collections in Open MIC dataset, we randomly select one museum as the training set and the other one as testing set to 
We evaluate the Relation Net \cite{sung2017learning} as a baseline and compare our SoSN model against it. 

\subsection{Experimental setup}
%Below, we detail the experimental setup for each dataset.

For the Omniglot dataset, %we randomly sample $C$ classes at first. 
follow the setup in \cite{sung2017learning} for which 5-way 1-shot problem contains 1 support and 19 query images, 5-way 5-shot problem contains 5 support and 10 query images, 20-way 1-shot problem contains 1 support and 10 query images, and 20-way 5-shot problem contains 1 support and 5 query images for each sampled class per episode. We sample 1 and 5 query images per class for 5-way 1-shot and 5-shot experiments, and 1 and 2 query image per class for 20-way 1-shot and 5-shot setup.

For the {\em mini}Imagenet dataset, we follow the standard 5-way 1-shot and 5-way 5-shot protocols. For every training and testing episode, we randomly select 5 and 3 query samples per class. We average over 600 episodes to obtain the performance of our SoSN model. The initial learning rate is set to $1e\!-\!3$. We train the model with $200000$ episodes.

For the Open MIC dataset, we enumerate all possible training--testing pairs from the 4 combined subproblems. For every selected subproblem, we performed the mean extraction. Then we center-cropped each image with ratio 0.875 and resized all images to $84\!\times\!84$. As some classes of Open MIC contain less than 5 images, we apply the 5-way 1-shot learning protocol. During training, to form an episode, we select 1 image for the support set and another 2 images for the query set from each class. During testing, we use the same number of support and query samples in every episode and compute the accuracy over 1000 episodes for the final score. The initial learning rate is set to $1e\!-\!4$. We use $50000$ episodes to train the models.

\begin{table*}[t]
\vspace{-0.3cm}
\centering
\caption{Evaluations on Omniglot. Refer to \cite{sung2017learning} for references to baselines listed in this table.}
\label{table1}
\vspace{-0.2cm}
\makebox[\textwidth]{
\setlength{\tabcolsep}{0.10em}
\renewcommand{\arraystretch}{0.70}
%\fontsize{8.5}{9}\selectfont
\begin{tabular}{l|c|c|c|c|c}
Model & Fine & \multicolumn{2}{c}{5-way Acc.} & \multicolumn{2}{c}{20-way Acc.} \\ 
& Tune & 1-shot & 5-shot & 1-shot  & 5-shot  \\ \hline
%\textit{MANN} & N & $82.8\%$ & $94.9\%$ & - & - \\
%\textit{Convolutional Siamese Nets} & N & $96.7\%$ & $98.4\%$ & $88.0\%$ & $96.5\%$ \\
%\textit{Convolutional Siamese Nets} & Y & $97.3\%$ & $98.4\%$ & $88.1\%$ & $97.0\%$ \\
%\textit{Matching Nets} & N & $98.1\%$ & $98.9\%$ & $93.8\%$ & $98.5\%$ \\
\textit{Matching Nets} & Y & $97.9\%$ & $98.7\%$ & $93.5\%$ & $98.7\%$ \\
\textit{Siamese Nets with Memory} & N & $98.4\%$ & $99.6\%$ & $95.0\%$ & $98.6\%$ \\
%\textit{Neural Statistician} & N & $98.1\%$ & $99.5\%$ & $93.2\%$ & $98.1\%$ \\ 
\textit{Meta Nets} & N & $99.0\%$ & - & $97.0\%$ & - \\ 
\textit{Prototypical Net} & N & 99.8\% & 99.7\% & 96.0\% & 98.9\% \\ 
%\textit{MAML} & Y & $98.7 \pm 0.4\%$ & $99.9 \pm 0.1\%$ & $95.8 \pm 0.3\%$ & $98.9 \pm 0.2\%$ \\ 
\textit{Relation Net} & N & $99.6 \pm 0.2\%$ & $99.8 \pm 0.1\%$ & $97.6 \pm 0.2\%$ & $99.1 \pm 0.1\%$ \\ 
\hline
\textit{SoSN($\otimes$)+AsinhE}, Eq. \eqref{eq:concat_best}+Rem. \ref{re:asinhe} & N & $99.7 \pm 0.2\%$ & $99.9 \pm 0.1\%$ & $98.2 \pm 0.2\%$ & $99.3 \pm 0.1\%$ \\
\textit{SoSN($\otimes$)+SigmE}, $\;$Eq. \eqref{eq:concat_best}+Rem. \ref{re:pnpn} & N & ${\bf 99.8 \pm 0.1\%}$ & ${\bf 99.9 \pm 0.1\%}$ & ${\bf 98.3 \pm 0.2\%}$ & ${\bf 99.4 \pm 0.1\%}$ \\ 
\hline
\end{tabular}}
\end{table*}
%\medskip

\subsection{Results}
Below we validate the effectiveness of our SoSN model and compare it against the state-of-the-art approaches.

\noindent\textbf{Omniglot.}  The evaluation results are shown in Table \ref{table1}. As shown in the table, despite the performance on the Omniglot dataset is saturated, our network can still outperform other few-shot learning methods on all evaluation protocols. Due to saturation, we use Omniglot for the basic sanity checks.
\begin{table*}[t]
\centering
\caption{Evaluations on the \textit{mini}Imagenet dataset. Refer to \cite{sung2017learning} for references to baselines.}
\label{table2}
\vspace{-0.2cm}
\makebox[\textwidth]{
\setlength{\tabcolsep}{0.10em}
\renewcommand{\arraystretch}{0.70}
%\fontsize{8.5}{9}\selectfont
\begin{tabular}{l|c|c|c}
Model & Fine Tune & \multicolumn{2}{c}{5-way Acc.} \\ 
& & 1-shot & 5-shot \\ \hline
\textit{Matching Nets} & N & $43.56 \pm 0.84\%$ & $55.31 \pm 0.73\%$  \\
%\textit{Meta Nets} & N & $49.21 \pm 0.96\%$ & - \\
\textit{Meta-Learn Nets} & N & $43.44 \pm 0.77\%$ & $60.60 \pm 0.71\%$ \\
\textit{Prototypical Net} & N & $49.42 \pm 0.78\%$ & $68.20 \pm 0.66\%$ \\ 
\textit{MAML} & Y & $48.70 \pm 1.84\%$ & $63.11 \pm 0.92\%$ \\ 
\textit{Relation Net} & N & $50.44 \pm 0.82\%$ & $65.32 \pm 0.70\%$  \\ 
\hline
\textit{SoSN($\otimes$+F)+SigmE}, $\;\;\;$Eq. \eqref{eq:concat_f} & N & $50.57 \pm 0.84\%$ & 65.91 $\pm$ 0.71\%  \\
\textit{SoSN($\otimes$+R)+SigmE}, $\;\;\;$Eq. \eqref{eq:concat_r} & N & $52.96\pm 0.83\%$ & $68.58 \pm 0.70\%$  \\ 
\textit{SoSN($\otimes$)} (no Power Norm.), Eq. \eqref{eq:concat_best}& N & $50.88 \pm 0.85\%$ & 66.71 $\pm$ 0.67\%  \\
\textit{SoSN($\otimes$)+AsinhE}, $\;$Eq. \eqref{eq:concat_best}+Remark \ref{re:asinhe} & N & $52.10 \pm 0.82\%$ & $67.79 \pm 0.69\%$ \\
\textit{SoSN($\otimes$)+SigmE}, $\;\;$Eq. \eqref{eq:concat_best}+Remark \ref{re:pnpn} & N & ${\bf 52.96\pm 0.83\%}$ & ${\bf68.63 \pm 0.68\%}$  \\ 
%\textit{SoSN($\otimes$\&$\oplus$)+SigmE}, $\;\;$Eq. \eqref{eq:oplus} & N & $52.94 \pm 0.83\%$ & $68.16 \pm 0.68\%$      \\
%\textit{SoSN($\otimes$\&$\ominus$)+SigmE}, $\;\;$Eq. \eqref{eq:ominus} & N & $52.80 \pm 0.84\%$ & $ 67.36\pm 0.70\%$      \\
%
%
%\textit{SoSN($\otimes$--$\odot$)+SigmE}, $\;\;\;$Eq. \eqref{eq:cosine} & N & $52.07 \pm 0.84\%$ & $ 66.38\pm 0.65\%$      \\
%\textit{SoSN($\otimes$--$\ominus$)+SigmE}, $\;\;\;$Eq. \eqref{eq:basic_diff} & N & $52.10 \pm 0.82\%$ & $66.60 \pm 0.68\%$ \\
\hline
\textit{SoSN($\otimes$)+Permtation+SigmE} & N & ${\bf 54.52\pm 0.86\%}$ & $ 68.52\pm 0.69\% $\\
\textit{SoSN($\otimes$)+Permtation+SigmE+Multiple-similarity Net.} & N & ${\bf 54.72\pm 0.89\%}$ & ${\bf 68.67\pm 0.67\%}$\\
\hline
\textit{SoSN($\otimes$)+SigmE} (224$\times$224) & N & ${\bf 58.41 \pm 0.87\%}$ & ${\bf 72.96 \pm 0.67\%}$  \\
\hline
\textit{SoSN($\otimes$)+Permtation+SigmE} (224$\times$224) & N & ${\bf 58.80\pm 0.90\%}$ & $ {\bf 73.11\pm 0.71\%} $\\%fix
\textit{SoSN($\otimes$)+Permtation+SigmE} (224$\times$224)\textit{+Mult.-sim. Net.} & N & ${\bf 59.22\pm 0.91\%}$ & ${\bf 73.24\pm 0.69\%}$\\%fix

\hline
\end{tabular}}
\vspace{-0.3cm}
\end{table*}

{%\small
\begin{table*}[b]
\vspace{-0.1cm}
%\centering
\caption{Evaluations on the Open MIC dataset.}
\vspace{-0.2cm}
\label{table5}
\makebox[\textwidth]{
\hspace{-0.3cm}
\setlength{\tabcolsep}{0.10em}
\renewcommand{\arraystretch}{0.70}
\fontsize{8.5}{9}\selectfont
\begin{tabular}{l|c|c|c|c|c|c|c|c|c|c|c|c|c}
Model & $L$ & $p1\!\!\rightarrow\!p2$ & $p1\!\!\rightarrow\!p3$& $p1\!\!\rightarrow\!p4$& $p2\!\!\rightarrow\!p1$& $p2\!\!\rightarrow\!p3$ &$p2\!\!\rightarrow\!p4$& $p3\!\!\rightarrow\!p1$& $p3\!\!\rightarrow\!p2$& $p3\!\!\rightarrow\!p4$& $p4\!\!\rightarrow\!p1$& $p4\!\!\rightarrow\!p2$& $p4\!\!\rightarrow\!p3$\\
\hline
Relation Net & 5 & $71.1 \!\pm\! 1.0$ & $53.6 \!\pm\! 1.1$ & $63.5 \!\pm\! 1.0$ & $47.2 \!\pm\! 1.0$ & $50.6 \!\pm\! 1.1$ & $68.5 \!\pm\! 1.0$ & $48.5 \!\pm\! 1.1$ & $49.7 \!\pm\! 1.1$ & $68.4 \!\pm\! 1.0$ & $45.5 \!\pm\! 1.0$ & $70.3 \!\pm\! 1.0$ & $50.8 \!\pm\! 1.1$\\
Relation Net & 20 & $40.1 \!\pm\! 0.5$ & $30.4 \!\pm\! 0.5$ & $41.4 \!\pm\! 0.5$ & $23.5 \!\pm\! 0.4$ & $26.4 \!\pm\! 0.5$ & $38.6 \!\pm\! 0.5$ & $26.2 \!\pm\! 0.4$ & $25.8 \!\pm\! 0.4$ & $46.3 \!\pm\! 0.5$ & $23.1 \!\pm\! 0.4$ & $43.3 \!\pm\! 0.5$ & $27.7 \!\pm\! 0.4$\\ \hline
SoSN & 5 & $80.8 \!\pm\! 0.9$ & $64.3 \!\pm\! 1.1$ & $74.9 \!\pm\! 1.1$ & $58.8 \!\pm\! 1.1$ & $61.2 \!\pm\! 1.1$ & $76.9 \!\pm\! 0.9$ & $61.3 \!\pm\! 1.1$ & $80.8 \!\pm\! 0.9$ & $77.2 \!\pm\! 1.0$ & $58.2 \!\pm\! 1.1$ & $80.1 \!\pm\! 0.9$ & $61.6 \!\pm\! 1.1$\\
SoSN & 20 & $57.0 \!\pm\! 0.5$ & $42.3 \!\pm\! 0.5$ & $54.2 \!\pm\! 0.5$ & $34.7 \!\pm\! 0.5$ & $37.0 \!\pm\! 0.5$ & $54.8 \!\pm\! 0.5$ & $36.0 \!\pm\! 0.5$ & $57.1 \!\pm\! 0.5$ & $57.0 \!\pm\! 0.5$ & $36.4 \!\pm\! 0.5$ & $59.3 \!\pm\! 0.9$ & $37.8 \!\pm\! 0.5$\\ \hline
SoSN+SigmE & 5 & $\mathbf{81.4} \!\pm\! 0.9$ & $\mathbf{65.2} \!\pm\! 1.1$ & $\mathbf{75.1} \!\pm\! 1.0$ & $\mathbf{60.3} \!\pm\! 1.1$ & $\mathbf{62.1} \!\pm\! 1.1$ & $\mathbf{77.7} \!\pm\! 0.9$ & $\mathbf{61.5} \!\pm\! 1.1$ & $\mathbf{82.0} \!\pm\! 1.0$ & $\mathbf{78.0} \!\pm\! 1.0$ & $\mathbf{59.0} \!\pm\! 1.1$ & $\mathbf{80.8} \!\pm\! 1.0$ & $\mathbf{62.5} \!\pm\! 1.1$\\
SoSN+SigmE & 20 &   $\mathbf{59.4} \!\pm\! 0.6$ & $\mathbf{42.5} \!\pm\! 0.5$ & $\mathbf{55.0} \!\pm\! 0.5$ & $\mathbf{35.1} \!\pm\! 0.5$ & $\mathbf{38.3} \!\pm\! 0.5$ & $\mathbf{56.3} \!\pm\! 0.5$ & $\mathbf{38.7} \!\pm\! 0.5$ & $\mathbf{57.9} \!\pm\! 0.6$ & $\mathbf{59.4} \!\pm\! 0.5$ & $\mathbf{37.4} \!\pm\! 0.5$ & $\mathbf{59.0} \!\pm\! 0.5$ & $\mathbf{38.6} \!\pm\! 0.5$\\
224x224 & 5 & $\mathbf{83.9} \!\pm\! 0.9$ & $\mathbf{68.9} \!\pm\! 1.1$ & $\mathbf{82.1} \!\pm\! 0.9$ & $\mathbf{64.7} \!\pm\! 1.1$ & $\mathbf{66.6} \!\pm\! 1.1$ & $\mathbf{82.2} \!\pm\! 0.9$ & $\mathbf{65.5} \!\pm\! 1.1$ & $\mathbf{84.5} \!\pm\! 0.8$ & $\mathbf{80.6} \!\pm\! 0.8$ & $\mathbf{64.6} \!\pm\! 1.1$ & $\mathbf{83.6} \!\pm\! 0.8$ & $\mathbf{66.0} \!\pm\! 1.1$\\
224x224 & 20 & $\mathbf{63.6} \!\pm\! 0.5$ & $\mathbf{46.7} \!\pm\! 0.6$ & $\mathbf{63.6} \!\pm\! 0.5$ & $\mathbf{39.0} \!\pm\! 0.5$ & $\mathbf{43.9} \!\pm\! 0.5$ & $\mathbf{61.8} \!\pm\! 0.5$ & $\mathbf{43.7} \!\pm\! 0.5$ & $\mathbf{63.3} \!\pm\! 0.5$ & $\mathbf{62.5} \!\pm\! 0.5$ & $\mathbf{42.7} \!\pm\! 0.5$ & $\mathbf{61.5} \!\pm\! 0.5$ & $\mathbf{43.7} \!\pm\! 0.5$\\

\hline
\end{tabular}}
p1: shn+hon+clv, p2: clk+gls+scl, p3: sci+nat, p4: shx+rlc. Notation {\em x$\!\rightarrow$y} means training on exhibition {\em x} and testing on {\em y}.
\vspace{-0.2cm}
\end{table*}
}
\noindent\textbf{{\em mini}Imagenet.} Table \ref{table2} demonstrates that our method outperforms other approaches on both evaluation protocols. For image size of $84\!\times\!84$ and 5-way 1-shot experiment, our SoSN model achieved $2.5\%$ higher accuracy than Relation Net \cite{sung2017learning}. %For 5-way 5-shot, Prototypical Net was the best one, however, it is trained with 20-way 15 queries in every episode. 
Our SoSN also outperforms Prototypical Net by $0.43\%$ accuracy on the 5-way 5-shot protocol. When we use $224\!\times\!224$ images to train SoSN, the accuracies on both protocols increase by $\sim\!5\%$ and $\sim\!4\%$, respectively, which illustrates that SoSN benefits from larger image sizes as its second-order matrices, defined in Table \ref{tab_methods}, become full-rank while the similarity net. needs no modifications. Moreover, Table \ref{tab_methods} shows that using 3 permutations fed to single and multiple similarity networks yields improvements by 1.8 and 2.06\% for $84\!\times\!84$ images SoSN with no permutation. We can also observe $\sim$1\% gain for $224\!\times\!224$ images.

\noindent\textbf{Open MIC.} Table \ref{table5} demonstrates that our SoSN model outperforms the Relation Net \cite{sung2017learning} for all train/test subproblems. For $5$- and $20$-Way, Relation Net scores 55.45 and 31.58\%. In contrast, our SoSN scores 68.23 and 45.31\%, resp. Lastly, our SoSN+SigmE scores 69.06 and 46.53\%, respectively. 
%
%{\em nat2gls} subproblem, the gain of SoSON to the baseline amounts to $14.7\%$. For $clv2nat$, the gain of SoSN to the baseline is $71.5\%$. 
The above large gains show that the second-order  relationship descriptor combined with the SigmE pooling is beneficial for the task of similarity learning. See our supplementary material for additional evaluations.

\noindent\textbf{Power Normalization and Relationship Descriptors. }  Below we discuss AsinhE vs. SigmE pooling from Remarks \ref{re:asinhe} and \ref{re:pnpn}; the latter operator closely approximating our statistically motivated MaxExp$\,$($\pm$) pooling from Eq. \eqref{eq:pos_neg_pool2} and \eqref{eq:pos_neg_pool3}. Table \ref{table2} shows that SigmE performs $\sim\!0.9\%$ better than AsinhE. Tables \ref{table2} and \ref{table5} show $1.2$--$4.5\%$ benefit of SigmE in SoSN($\otimes$)+SigmE over SoSN($\otimes$) without it. This is consistent with \cite{koniusz2018deeper} and our statistical motivations. % for pooling. 

Relationship descriptors from Table \ref{tab_methods} are evaluated in Table \ref{table2} which shows that SoSN($\otimes$) outperforms Full Auto-correlation SoSN($\otimes$+F), Auto-correlation+concat. without average over $Z$ support feature vectors denoted as SoSN($\otimes$+R), and other variants. We expect that averaging over $Z$ support descriptors from $Z$ support images removes uncertainty in few-shot statistics while outer-products of support and query datapoints still benefit from spatially large conv. feature maps  and yield robust second-order statistics. This view is further supported by results for SoSN($\otimes$)+SigmE ($224\!\times\!224$).

\noindent\textbf{Permutations of Second-order Matrices. }  Table~\ref{table2}shows that stacking permuted second-order matrices yields up to $\sim$2\% improvement for $1$- and $5$-shot learning. We expect that, as second-order matrices contain correlation patterns globally, this simple strategy must benefit the similarity network which uses local patches/filters in convolutional layers. For more results \eg, w.r.t. numbers of permuted second-order matrices, kindly see our suppl. material.
\iftrue
{%\small
\begin{table*}[b]
\vspace{-0.1cm}
%\centering
\caption{Evaluations on the Open MIC dataset for Protocol II.}
\vspace{-0.2cm}
\label{table5}
\makebox[\textwidth]{
\hspace{-0.3cm}
\setlength{\tabcolsep}{0.10em}
\renewcommand{\arraystretch}{0.70}
\fontsize{8.5}{9}\selectfont
\begin{tabular}{l|c|c|c|c|c|c|c|c|c|c|c}
Model & $ L $ & $shn$ & $hon$& $clv$& $clk$& $gls$ &$scl$& $sci$& $nat$& $shx$& $rlc$ \\
\hline
Relation Net & 5 & $43.2 \!\pm\! 1.0$ & $49.6 \!\pm\! 1.0$ & $49.8 \!\pm\! 1.0$ & $62.1 \!\pm\! 1.1$ & $59.3 \!\pm\! 1.0$ & $51.5 \!\pm\! 1.0$ & $45.9 \!\pm\! 1.0$ & $54.8 \!\pm\! 1.0$ & $71.1 \!\pm\! 1.0$ & $72.0 \!\pm\! 1.0$ \\
Relation Net & 20 & $20.8 \!\pm\! 0.4$ & $25.7 \!\pm\! 0.4$ & $26.1 \!\pm\! 0.4$ & $34.3 \!\pm\! 0.4$ & $35.5 \!\pm\! 0.5$ & $18.4 \!\pm\! 0.3$ & $18.6 \!\pm\! 0.3$ & $32.8 \!\pm\! 0.5$ & $51.8 \!\pm\! 0.5$ & $48.2 \!\pm\! 0.5$ \\
Relation Net & 30 & $18.1 \!\pm\! 0.3$ & $21.1 \!\pm\! 0.3$ & $23.2 \!\pm\! 0.3$ & $27.0 \!\pm\! 0.3$ & $31.8 \!\pm\! 0.4$ & $12.8 \!\pm\! 0.2$ & $12.4 \!\pm\! 0.2$ & $27.1 \!\pm\! 0.3$ & $40.6 \!\pm\! 0.4$ & $41.0 \!\pm\! 0.4$ \\ \hline
SoSN & 5 & $60.3 \!\pm\! 1.1$ & $62.6 \!\pm\! 1.1$ & $60.5 \!\pm\! 1.1$ & $72.9 \!\pm\! 1.1$ & $74.3 \!\pm\! 1.1$ & $72.3 \!\pm\! 1.0$ & $53.4 \!\pm\! 1.1$ & $68.0 \!\pm\! 1.1$ & $77.0 \!\pm\! 1.0$ & $78.4 \!\pm\! 1.0$ \\
SoSN & 20 & $36.3 \!\pm\! 0.5$ & $34.4 \!\pm\! 0.5$ & $32.3 \!\pm\! 0.4$ & $46.5 \!\pm\! 0.5$ & $48.6 \!\pm\! 0.5$ & $50.1 \!\pm\! 0.5$ & $24.8 \!\pm\! 0.4$ & $42.0 \!\pm\! 0.5$ & $58.5 \!\pm\! 0.5$ & $53.2 \!\pm\! 0.5$ \\
SoSN & 30 & $32.0 \!\pm\! 0.4$ & $32.6 \!\pm\! 0.4$ & $27.7 \!\pm\! 0.3$ & $44.3 \!\pm\! 0.4$ & $44.0 \!\pm\! 0.4$ & $44.2 \!\pm\! 0.4$ & $16.3 \!\pm\! 0.3$ & $39.9 \!\pm\! 0.4$ & $53.7 \!\pm\! 0.4$ & $47.8 \!\pm\! 0.4$  \\ \hline
SoSN+SigmE & 5 & $61.5 \!\pm\! 1.1$ & $63.6 \!\pm\! 1.1$ & $61.7 \!\pm\! 1.1$ & $74.5 \!\pm\! 1.2$ & $74.9 \!\pm\! 1.1$ & $72.9 \!\pm\! 1.0$ & $54.2 \!\pm\! 1.0$ & $68.9 \!\pm\! 1.1$ & $78.0 \!\pm\! 1.0$ & $79.1 \!\pm\! 1.0$ \\
SoSN+SigmE  & 20 & $37.4 \!\pm\! 0.5$ & $37.5 \!\pm\! 0.5$ & $33.9 \!\pm\! 0.4$ & $48.6 \!\pm\! 0.5$ & $53.2 \!\pm\! 0.5$ & $50.5 \!\pm\! 0.5$ & $25.1 \!\pm\! 0.4$ & $45.3 \!\pm\! 0.5$ & $58.9 \!\pm\! 0.5$ & $56.6 \!\pm\! 0.5$ \\
SoSN+SigmE  & 30 & $32.8 \!\pm\! 0.4$ & $33.4 \!\pm\! 0.4$ & $30.6 \!\pm\! 0.3$ & $45.3 \!\pm\! 0.5$ & $49.9 \!\pm\! 0.4$ & $49.3 \!\pm\! 0.3$ & $21.1 \!\pm\! 0.3$ & $40.8 \!\pm\! 0.4$ & $54.9 \!\pm\! 0.4$ & $48.7 \!\pm\! 0.5$  \\
\hline
\end{tabular}}
\centering Training on source images and testing on target images for every exhibition respectively.
\vspace{-0.2cm}
\end{table*}
}
\fi

%% file: conclusions.tex
\ifdefined\arxiv
\newcommand{\PowH}{3.0cm}
\newcommand{\PowHB}{2.875cm}
\newcommand{\PowW}{3.65cm}
\else
\newcommand{\PowH}{3.2cm}
\newcommand{\PowHB}{3.4cm}
\newcommand{\PowW}{4.5cm}
\fi

\begin{figure*}[t]%htbp % left bottom right top
\centering
\vspace{-0.2cm}
\hspace{-0.45cm}
\begin{subfigure}[t]{0.32\linewidth}
\centering\includegraphics[trim=0 0 0 0, clip=true, height=\PowH]{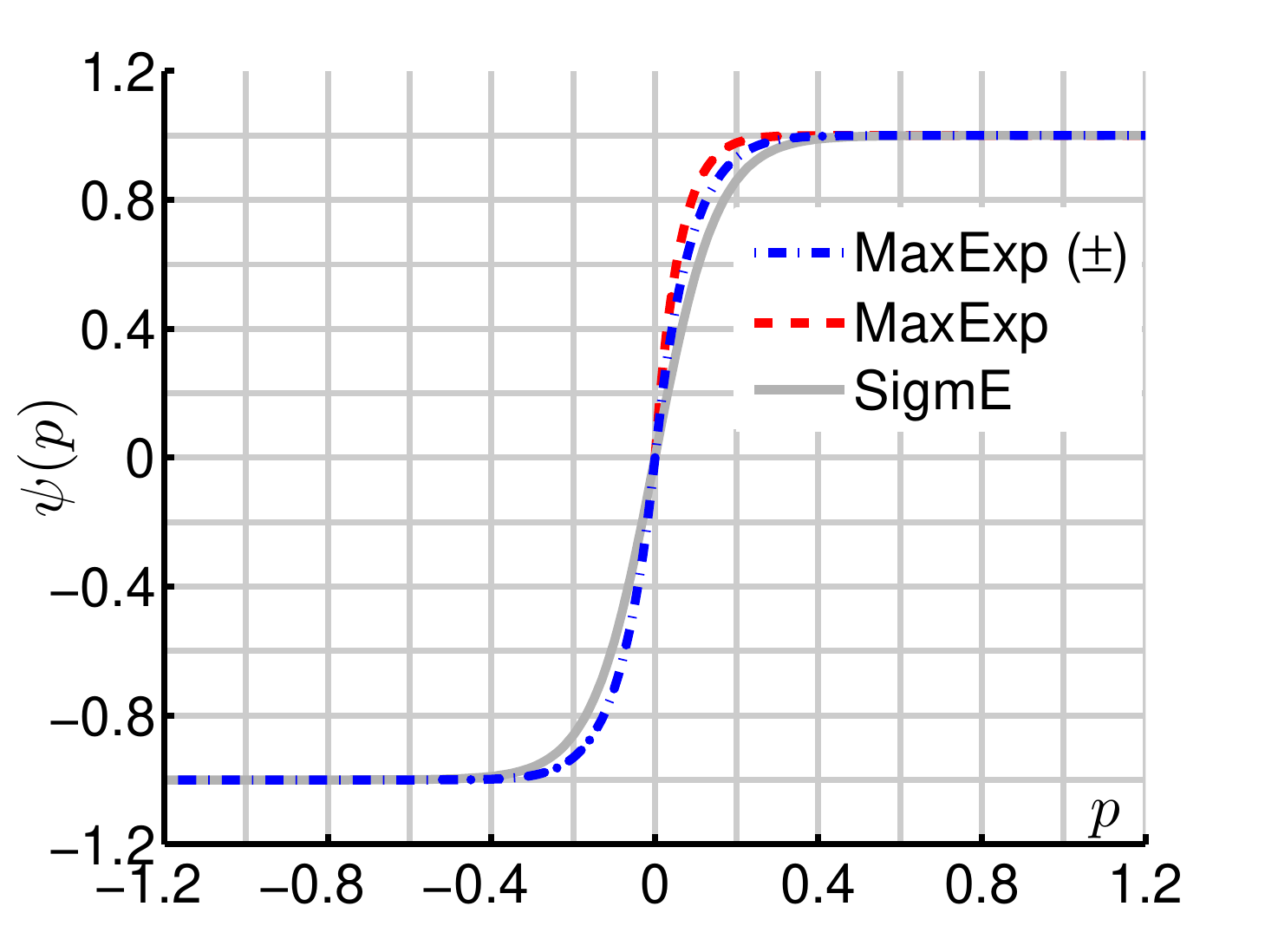}\vspace{-0.2cm}
\caption{\label{fig:pow1}}
\end{subfigure}
\begin{subfigure}[t]{0.32\linewidth}
\centering\includegraphics[trim=0 0 0 0, clip=true, height=\PowH]{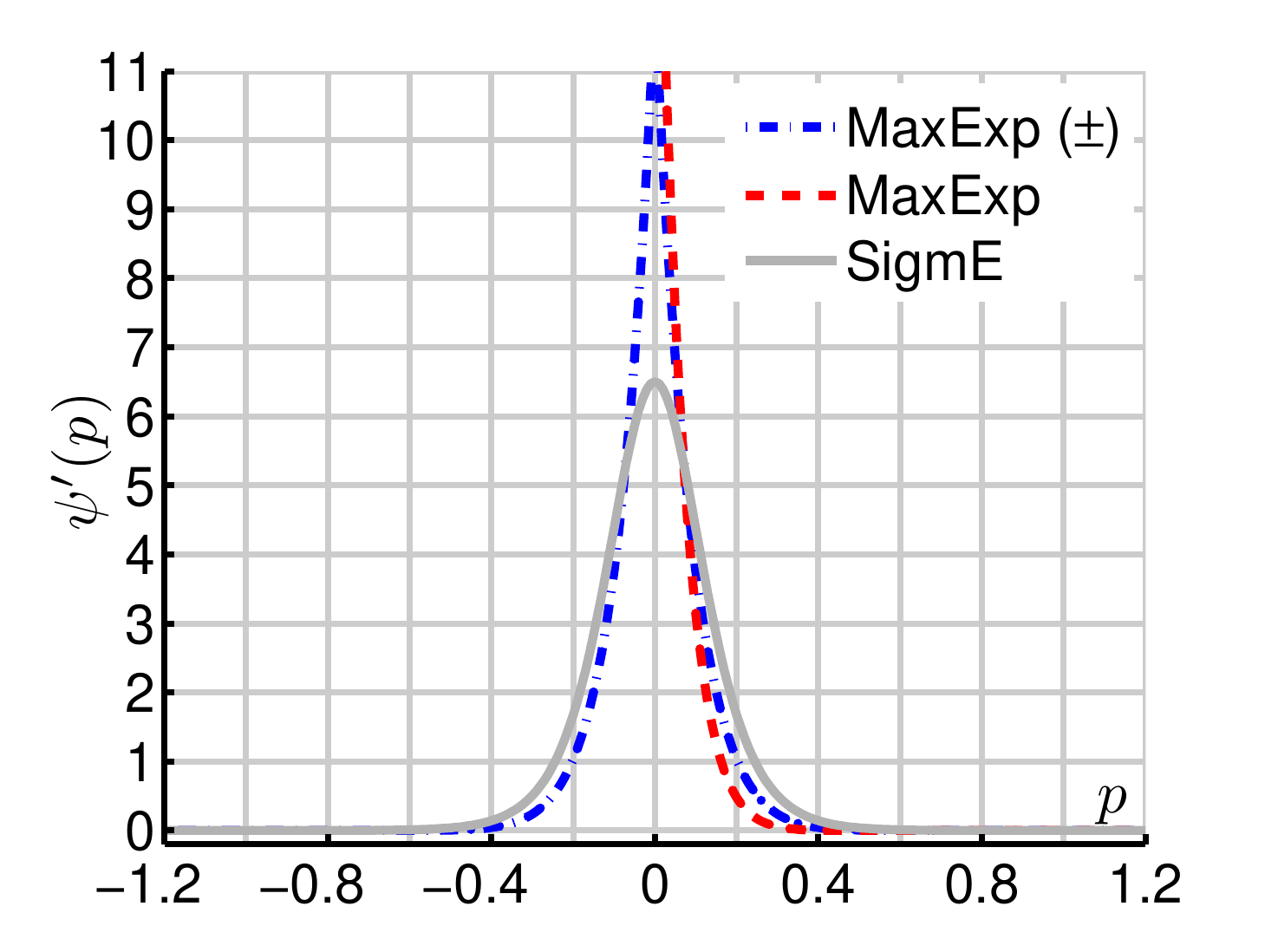}\vspace{-0.2cm}
\caption{\label{fig:pow2}}
\end{subfigure}
\begin{subfigure}[t]{0.32\linewidth}
\centering\includegraphics[trim=0 0 0 0, clip=true, height=\PowH]{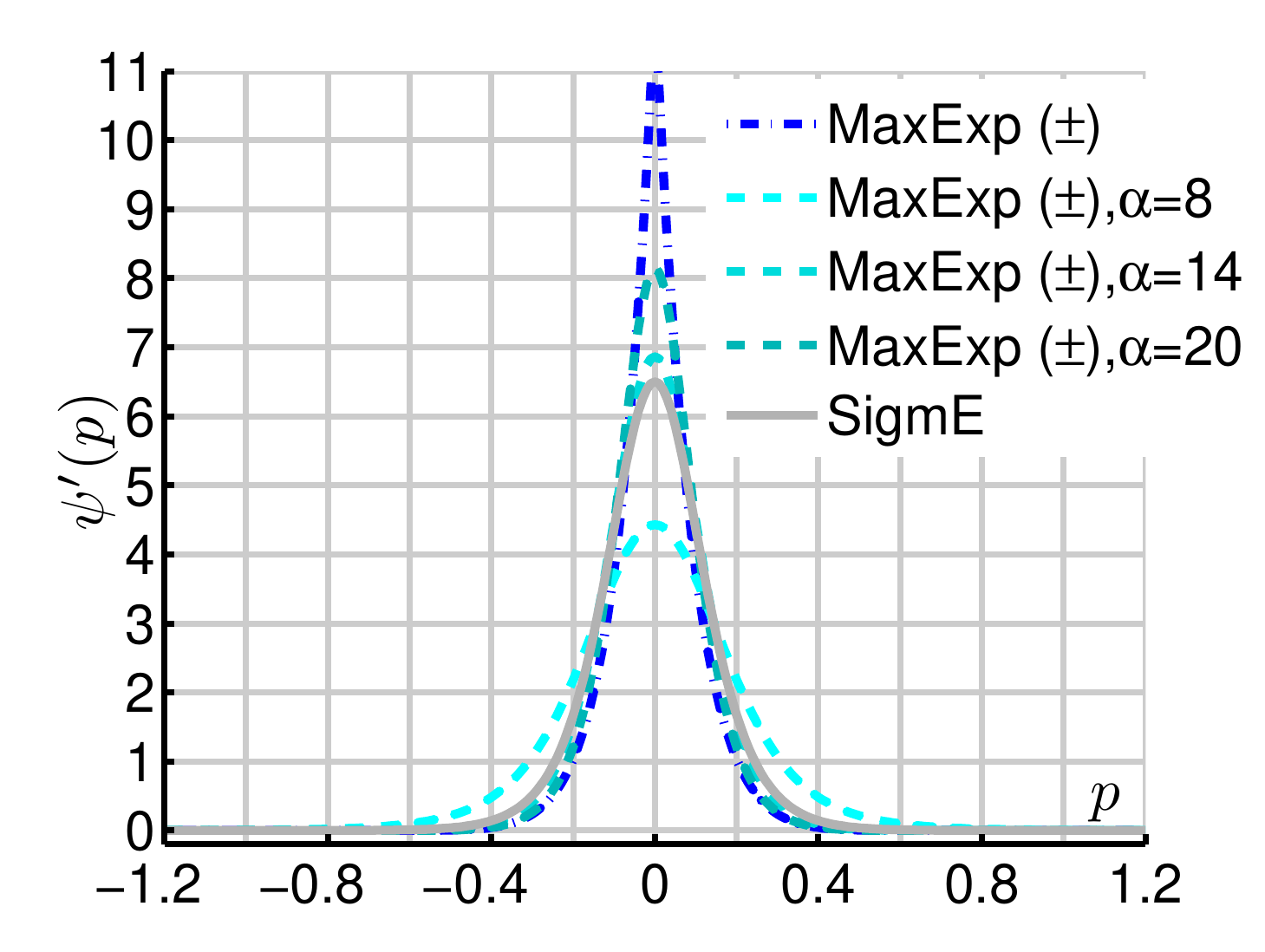}\vspace{-0.2cm}
\caption{\label{fig:pow3}}
\end{subfigure}

%\begin{subfigure}[t]{0.24\linewidth}
%\centering\includegraphics[trim=0 0 0 0, clip=true, height=\PowH, width=\PowW]{images/pow4.pdf}\vspace{-0.2cm}
%\caption{\label{fig:pow4}}
%\end{subfigure}

\vspace{-0.2cm}
\caption{MaxExp$\,$($\pm$), MaxExp, and SigmE are illustrated in Figure \ref{fig:pow1}. Note that MaxExp$\,$($\pm$) extends MaxExp to negative values and is closely approximated by SigmE. Derivatives of MaxExp$\,$($\pm$), MaxExp, and SigmE are in Figure \ref{fig:pow2} while derivatives of MaxExp$\,$($\pm$) with the soft maximum are demonstrated to be smooth in Figure \ref{fig:pow3}.}
\vspace{-0.3cm}
\label{fig:power-norms}
\end{figure*}

\section{Conclusions}
\label{sec:conclusions}

In this paper, we present an end-to-end trainable SoSN model leveraging second-order statistics and Power Normalization to learn the similarity between images in the context of few-shot learning. We investigate how to capture relations between the query and support second-order matrices; including our permutation-based representation. We present a statistically motivated interpretation of Power Normalization and investigate the influence of different operations and image sizes on results %of second-order relationship descriptors. We
and obtain state-of-the-art results on Omniglot and \textit{mini}Imagenet. %and on all protocols. 
In addition, we introduce  one-shot learning protocol for the Open MIC dataset and show that SoSN outperforms the baseline.

%% file: acknowledgement.tex
\section*{Acknowledgement}
This work is partially supported by CSIRO Scientific Computing, and the China Scholarship Council (CSC Student ID 201603170283).
\vspace{-0.3cm}

%% file: appendix.tex
\vspace{-0.1cm}
\section{Deriving/Interpreting Pooling via Sigmoid}
\label{sec:pool_der}

Proposition \ref{pr:cooc} and Remark \ref{re:maxexp} state that quantity $1\!-\!(1\!-\!p)^N$ is the probability of at least one co-occurrence being detected in the pool of the $N$ i.i.d. trials performed according to the Bernoulli distribution with the success probability $p$ of event $(\phi_{n}\!\cap\!\phi'_{n}\!=\!1)$ for event vectors in $\vphi,\vphi'\!\!\in\!\{0,1\}^{N}$. Below we extend this theory to negative co-occurrences which we interpret as two anti-correlating visual words.%the case of co-occurrences.

\begin{proposition}
\label{pr:cooc_anti}
Assume an event vector $\vphi^*\!\!\in\!\{0,1,-1\}^{N}$ constructed from event vectors $\vphi,\vphi'\!\!\in\!\{0,1,-1\}^{N}$, which stores the $N$ trials performed according to the Multinomial distribution under i.i.d. assumption, 
 for which the probability $p$ of an event $(\phi^*\!\!=\!\phi_{n}\!\cdot\phi'_{n}\!=\!1)$ denotes a co-occurrence, the probability  $q$ of an event  $(\phi^*\!\!=\phi_{n}\!\cdot\phi'_{n}\!=\!-1)$ denotes a negatively-correlating co-occurrence, and $1\!-\!p\!-\!q$, for $(\phi^*\!\!=\phi_{n}\!\cdot\phi'_{n}\!=0)$ denotes the lack of the first two events, and $p$ is estimated as an expected value $p\!=\!\avg_n\phi^*_n$. Then the probability of at least one co-occurrence event $(\phi^*\!\!=\!1)$ minus the probability of at least one negatively-correlating co-occurrence event $(\phi^*\!\!=\!-1)$ in $N$ trials, encoded in $\phi^*_n$, becomes:
\begin{equation}
\label{eq:my_maxexp3}
\psi\!=\!(1\!-\!q)^{N}\!-\!(1\!-\!p)^{N}.
%\vspace{-0.1cm}
\end{equation}
\end{proposition}
\begin{proof}
One can derive the above difference of probabilities by directly applying the Multinomial calculus as follows:
\begin{align}
& \!\!\!\!\!\!\!\textstyle\sum\limits_{n=1}^{N}\sum\limits_{n'=0}^{N\!-n}\!\!\binom{N}{n,n'\!,N\!-n-n'\!-n''\!}\left(p^nq^{n'\!}\!-\!p^{n'\!}q^{n}\right)(1\!\!-\!\!p\!\!-\!\!q)^{N\!-n-n'}\!.%\nonumber\\[-10pt]
%&
\label{eq:my_maxexppr2}
\end{align}
One can verify algebraically/numerically that Eq. \eqref{eq:my_maxexppr2} and \eqref{eq:my_maxexp3} are equivalent.% which completes the proof.
\end{proof}

\begin{remark}
%\vspace{-0.3cm}
\label{re:pos_neg_pool}
If we decide from the beginning that the $N$ trials contain only either co-occurrences with the event probability $p\!=\!\rho p'\!$ or negatively-correlating co-occurrences with the event probability $q\!=\!(1\!-\!\rho)q'\!$, we can set $q\!:=\!(1\!-\!\rho)\max(0,-p)$ and $p\!:=\!\rho\max(0,p)$, which leads to:
\begin{align}
& \psi\!=\!\left(1\!-\!(1\!-\!\rho)\max(0,-p)\right)^N \!\!-\left(1\!-\!\rho\max(0,p)\right)^N\!\!\!,
\label{eq:pos_neg_pool2}
\end{align}
where $0\!\leq\rho\!\leq\!1$, $-1\!\leq\!p\!\leq\!1$, while $\rho p\!\geq\!0$ and $(1\!-\!\rho) p\!<\!0$ encode the probability of event of co-occurrence and negatively-correlating co-occurrence, respectively.
\end{remark}
\begin{remark}
%\vspace{-0.3cm}
\label{re:pos_neg_pool2}
In practice, we trace-normalize $\mM$, use reg. $\lambda\!\approx\!1e\!-\!6$ and $0\!<\!\eta\!\approx\!N$ following Remarks \ref{re:pnpn} and \ref{re:maxexp}:
\begin{align}
& \!\!\!\!\mPsi\!=\!\mygthree{\,\mM,\eta\,}\!=\!\left(1-(1\!-\!\rho)\max(0,\frac{-\mM}{\trace(\mM)\!+\!\lambda})\right)^\eta \!\!\nonumber\\
&\qquad\qquad\qquad\quad-\left(1-\rho\max(0,\frac{\mM}{\trace(\mM)\!+\!\lambda})\right)^\eta\!\!\!.
\label{eq:pos_neg_pool3}
\end{align}
We also apply the soft maximum funct. $\max(\mX,\mY; \alpha)=\frac{1}{\alpha}\log\left(\exp(\alpha \mX)+\exp(\alpha \mY)\right)$ which role is to make derivatives of Eq. \eqref{eq:pos_neg_pool2} and \eqref{eq:pos_neg_pool3} smooth. Lastly, $\alpha$ controls the softness of the max function and $\alpha\!\gg\!\eta$. We call the above operator as MaxExp$\,$({$\pm$}) %$\,$\circled{$\pm$} 
in contrast to MaxExp in Eq. \eqref{eq:my_maxexp3}.
\end{remark}

Figure \ref{fig:power-norms} demonstrates that Eq. \eqref{eq:pos_neg_pool2} and \eqref{eq:pos_neg_pool3} are closely approximated by SigmE from Eq. \eqref{eq:sigmoid} for $\rho\!=\!0.5$. Moreover, derivative of MaxExp$\,$({$\pm$}) with soft maximum is smooth which is essential in back-propagation (non-smooth objectives do not guarantee convergence of the majority of optimization algorithms). As SigmE approximates closely our MaxExp$\,$({$\pm$}), we use SigmE in our experiments. We assert that our relationship descriptors/operators $\vartheta$ benefit from Power Normalization which, according to Prop. \ref{pr:cooc_anti}, can detect co-occurring responses of CNN filters and discards quantities of such responses which correlate with repeatable visual stimuli (\eg~variable areas of textures) which otherwise would introduce nuisance variability into representations. This nuisance variability would be further amplified when learning \eg~pair-wise relations as pairs of images would introduce nuisance quantities $\varepsilon$ and $\varepsilon^*\!$ that would result in a total nuisance variability of $\varepsilon\varepsilon^*\!\!\gg\!\max(\varepsilon,\varepsilon^*\!)$.

%
%% and $N$ needs to be perhaps 1.5--2$\times$ bigger than we typically set with MaxExp if $\rho\!\approx\!0.5$. $\lambda\!\approx\!1e\!-\!6$?

%% file: camera_ready.bbl
\begin{thebibliography}{10}\itemsep=-1pt

\bibitem{akata2013label}
Z.~Akata, F.~Perronnin, Z.~Harchaoui, and C.~Schmid.
\newblock Label-embedding for attribute-based classification.
\newblock {\em CVPR}, pages 819--826, 2013.

\bibitem{BartU05}
E.~Bart and S.~Ullman.
\newblock Cross-generalization: Learning novel classes from a single example by
  feature replacement.
\newblock {\em CVPR}, pages 672--679, 2005.

\bibitem{boureau_midlevel}
Y.~Boureau, F.~Bach, Y.~LeCun, and J.~Ponce.
\newblock {Learning Mid-Level Features for Recognition}.
\newblock {\em CVPR}, 2010.

\bibitem{boureau_pooling}
Y.~Boureau, J.~Ponce, and Y.~LeCun.
\newblock {A Theoretical Analysis of Feature Pooling in Vision Algorithms}.
\newblock {\em ICML}, 2010.

\bibitem{carreira_secord}
J.~Carreira, R.~Caseiro, J.~Batista, and C.~Sminchisescu.
\newblock {Semantic Segmentation with Second-Order Pooling.}
\newblock {\em ECCV}, 2012.

\bibitem{farhadi2009describing}
A.~Farhadi, I.~Endres, D.~Hoiem, and D.~Forsyth.
\newblock Describing objects by their attributes.
\newblock {\em CVPR}, pages 1778--1785, 2009.

\bibitem{fei2006one}
L.~Fei-Fei, R.~Fergus, and P.~Perona.
\newblock One-shot learning of object categories.
\newblock {\em PAMI}, 28(4):594--611, 2006.

\bibitem{NIPS2004_2576}
M.~Fink.
\newblock Object classification from a single example utilizing class relevance
  metrics.
\newblock {\em NIPS}, pages 449--456, 2005.

\bibitem{finn2017model}
C.~Finn, P.~Abbeel, and S.~Levine.
\newblock Model-agnostic meta-learning for fast adaptation of deep networks.
\newblock In {\em ICML}, pages 1126--1135, 2017.

\bibitem{guo2013action}
K.~Guo, P.~Ishwar, and J.~Konrad.
\newblock Action recognition from video using feature covariance matrices.
\newblock {\em Trans. Img. Proc.}, 22(6):2479--2494, 2013.

\bibitem{Mehrtash_CVPR_2018}
M.~Harandi, M.~Salzmann, and R.~Hartley.
\newblock Joint dimensionality reduction and metric learning: A geometric take.
\newblock {\em ICML}, page 1404–1413, 2017.

\bibitem{jegou_bursts}
H.~J\'egou, M.~Douze, and C.~Schmid.
\newblock {On the Burstiness of Visual Elements}.
\newblock {\em CVPR}, pages 1169--1176, 2009.

\bibitem{koch2015siamese}
G.~Koch, R.~Zemel, and R.~Salakhutdinov.
\newblock Siamese neural networks for one-shot image recognition.
\newblock In {\em ICML Deep Learning Workshop}, volume~2, 2015.

\bibitem{koniusz2016tensor}
P.~Koniusz, A.~Cherian, and F.~Porikli.
\newblock Tensor representations via kernel linearization for action
  recognition from 3d skeletons.
\newblock In {\em ECCV}, pages 37--53. Springer, 2016.

\bibitem{koniusz2017domain}
P.~Koniusz, Y.~Tas, and F.~Porikli.
\newblock Domain adaptation by mixture of alignments of second-or higher-order
  scatter tensors.
\newblock In {\em CVPR}, volume~2, 2017.

\bibitem{me_museum}
P.~Koniusz, Y.~Tas, H.~Zhang, M.~Harandi, F.~Porikli, and R.~Zhang.
\newblock Museum exhibit identification challenge for the supervised domain
  adaptation.
\newblock {\em CoRR:1802.01093}, 2018.

\bibitem{me_tensor_tech_rep}
P.~Koniusz, F.~Yan, P.~Gosselin, and K.~Mikolajczyk.
\newblock {Higher-order Occurrence Pooling on Mid- and Low-level Features:
  Visual Concept Detection}.
\newblock {\em Technical Report}, 2013.

\bibitem{koniusz2017higher}
P.~Koniusz, F.~Yan, P.-H. Gosselin, and K.~Mikolajczyk.
\newblock Higher-order occurrence pooling for bags-of-words: Visual concept
  detection.
\newblock {\em PAMI}, 39(2):313--326, 2017.

\bibitem{me_ATN}
P.~Koniusz, F.~Yan, and K.~Mikolajczyk.
\newblock {Comparison of Mid-Level Feature Coding Approaches And Pooling
  Strategies in Visual Concept Detection}.
\newblock {\em CVIU}, 2012.

\bibitem{koniusz2018deeper}
P.~Koniusz, H.~Zhang, and F.~Porikli.
\newblock A deeper look at power normalizations.
\newblock In {\em CVPR}, pages 5774--5783, 2018.

\bibitem{kissme}
M.~Köstinger, M.~Hirzer, P.~Wohlhart, P.~M. Roth, and H.~Bischof.
\newblock Large scale metric learning from equivalence constraints.
\newblock {\em CVPR}, pages 2288--2295, 2012.

\bibitem{lake_oneshot}
B.~M. Lake, R.~Salakhutdinov, J.~Gross, and J.~B. Tenenbaum.
\newblock One shot learning of simple visual concepts.
\newblock {\em CogSci}, 2011.

\bibitem{larochelle2008zero}
H.~Larochelle, D.~Erhan, and Y.~Bengio.
\newblock Zero-data learning of new tasks.
\newblock {\em AAAI}, 1(2):3, 2008.

\bibitem{Li9596}
F.~F. Li, R.~VanRullen, C.~Koch, and P.~Perona.
\newblock Rapid natural scene categorization in the near absence of attention.
\newblock {\em Proceedings of the National Academy of Sciences},
  99(14):9596--9601, 2002.

\bibitem{bilinear_finegrained}
T.-Y. Lin, A.~R. Chowdhury, and S.~Maji.
\newblock Bilinear cnn models for fine-grained visual recognition.
\newblock {\em ICCV}, 2017.

\bibitem{miller_one_example}
E.~G. Miller, N.~E. Matsakis, and P.~A. Viola.
\newblock Learning from one example through shared densities on transforms.
\newblock {\em CVPR}, 1:464--471, 2000.

\bibitem{porikli2006tracker}
F.~Porikli and O.~Tuzel.
\newblock Covariance tracker.
\newblock {\em CVPR}, 2006.

\bibitem{book_nip}
J.~C. Rajapakse and L.~Wang.
\newblock {\em Neural Information Processing: Research and Development}.
\newblock Springer-Verlag Berlin and Heidelberg GmbH \& Co. KG, 2004.

\bibitem{romera2015embarrassingly}
B.~Romera-Paredes and P.~Torr.
\newblock An embarrassingly simple approach to zero-shot learning.
\newblock {\em ICML}, pages 2152--2161, 2015.

\bibitem{elbcm_brod}
A.~Romero, M.~Y. Ter{\'{a}}n, M.~Gouiff{\`{e}}s, and L.~Lacassagne.
\newblock Enhanced local binary covariance matrices {(ELBCM)} for texture
  analysis and object tracking.
\newblock {\em MIRAGE}, pages 10:1--10:8, 2013.

\bibitem{ILSVRC15}
O.~Russakovsky, J.~Deng, H.~Su, J.~Krause, S.~Satheesh, S.~Ma, Z.~Huang,
  A.~Karpathy, A.~Khosla, M.~Bernstein, A.~C. Berg, and L.~Fei-Fei.
\newblock {ImageNet} large scale visual recognition challenge.
\newblock {\em IJCV}, 115(3):211--252, 2015.

\bibitem{NIPS2017_7082}
A.~Santoro, D.~Raposo, D.~G. Barrett, M.~Malinowski, R.~Pascanu, P.~Battaglia,
  and T.~Lillicrap.
\newblock A simple neural network module for relational reasoning.
\newblock {\em NIPS}, pages 4967--4976, 2017.

\bibitem{deep_cooc}
Y.-F. Shih, Y.-M. Yeh, Y.-Y. Lin, M.-F. Weng, Y.-C. Lu, and Y.-Y. Chuang.
\newblock Deep co-occurrence feature learning for visual object recognition.
\newblock {\em CVPR}, 2017.

\bibitem{snell2017prototypical}
J.~Snell, K.~Swersky, and R.~Zemel.
\newblock Prototypical networks for few-shot learning.
\newblock In {\em NIPS}, pages 4077--4087, 2017.

\bibitem{sung2017learning}
F.~Sung, Y.~Yang, L.~Zhang, T.~Xiang, P.~H. Torr, and T.~M. Hospedales.
\newblock Learning to compare: Relation network for few-shot learning.
\newblock {\em CoRR:1711.06025}, 2017.

\bibitem{tuzel_rc}
O.~Tuzel, F.~Porikli, and P.~Meer.
\newblock Region covariance: {A} fast descriptor for detection and
  classification.
\newblock {\em ECCV}, 2006.

\bibitem{vinyals2016matching}
O.~Vinyals, C.~Blundell, T.~Lillicrap, D.~Wierstra, et~al.
\newblock Matching networks for one shot learning.
\newblock In {\em NIPS}, pages 3630--3638, 2016.

\bibitem{wang2011tracking}
Q.~Wang, F.~Chen, and W.~Xu.
\newblock Tracking by third-order tensor representation.
\newblock {\em Systems, Man, and Cybernetics, Part B: Cybernetics, IEEE
  Transactions on}, 41(2):385--396, 2011.

\bibitem{metric_old}
K.~Q. Weinberger, J.~Blitzer, and L.~K. Saul.
\newblock Distance metric learning for large margin nearest neighbor
  classification.
\newblock {\em NIPS}, pages 1473--1480, 2006.

\bibitem{woodworth_particle}
R.~S. Woodworth and E.~L. Thorndike.
\newblock The influence of improvement in one mental function upon the
  efficiency of other functions.
\newblock {\em Psychological Review (I)}, 8(3):247--261, 1901.

\bibitem{xian2017feature}
Y.~Xian, T.~Lorenz, B.~Schiele, and Z.~Akata.
\newblock Feature generating networks for zero-shot learning.
\newblock {\em CoRR:1712.00981}, 2017.

\bibitem{zhang2018model}
H.~Zhang and P.~Koniusz.
\newblock Model selection for generalized zero-shot learning.
\newblock {\em CoRR:1811.03252}, 2018.

\bibitem{zhang2018zero}
H.~Zhang and P.~Koniusz.
\newblock Zero-shot kernel learning.
\newblock {\em CVPR}, pages 7670--7679, 2018.

\end{thebibliography}
